\documentclass[letterpaper]{article}
\usepackage[preprint]{aaai2027}
\usepackage{times}
\usepackage{helvet}
\usepackage{courier}
\usepackage[hyphens]{url}
\usepackage{graphicx}
\usepackage{makecell}
\usepackage{array}
\usepackage{enumitem}
\usepackage{multirow}
\usepackage{subcaption}
\usepackage{graphicx}
\usepackage{tabularx}
\usepackage{color}
\usepackage{xspace} 
\usepackage{pifont}
\usepackage[utf8]{inputenc}
\usepackage{tcolorbox}
\usepackage{tabularx}
\usepackage{enumitem}
\usepackage{diagbox}
\usepackage[flushleft]{threeparttable}
\usepackage{marvosym}
\usepackage{amssymb}
\usepackage{amsmath}
\usepackage{amsthm}
\newtheorem{proposition}{Proposition}

\usepackage{booktabs}
\usepackage{xltabular}
\usepackage{tcolorbox}
\tcbuselibrary{breakable}
\usepackage{xcolor,colortbl}

\definecolor{fbApp}{HTML}{ffe4e3}
\definecolor{mydarkblue}{rgb}{0,0.3,0.9}

\newcommand{\first}[1]{\textbf{#1}}
\newcommand{\second}[1]{\underline{#1}}

\usepackage{natbib}
\usepackage{caption}
\usepackage{algorithm}
\usepackage{algorithmic}

\usepackage{newfloat}
\usepackage{listings}
\DeclareCaptionStyle{ruled}{labelfont=normalfont,labelsep=colon,strut=off}
\floatstyle{ruled}
\newfloat{listing}{tb}{lst}{}
\floatname{listing}{Listing}
\title{Contrastive Reinforced Policy Optimization via Privileged Self-Distillation}
\author{
    Xingjian Wu\equalcontrib,
    Junlin Liu\equalcontrib,
    Xingchen Liu\equalcontrib,
    Xuhang Zhu, Jianing Wang, \\ Linsen Guo\corresponding, Xiaoyu Li, Xuezhi Cao, Xunliang Cai
}
\affiliations{
    Meituan\\
    \{wuxingjian02, guolinsen\}@meituan.com
}

\begin{document}

\maketitle

\begin{abstract}
Recent advances in post-training Large Language Models (LLMs) increasingly rely on Reinforcement Learning with Verifiable Rewards (RLVR) or On-Policy Self-Distillation (OPSD).
While OPSD provides dense, logit-level supervision, it inherently suffers from exposure bias due to the privileged information of the self-teacher. 
In multi-turn agentic settings, this leads to reasoning route convergence and the loss of clear optimization directions. 
To tackle these challenges, we introduce Contrastive Reinforced Policy Optimization (CRPO), which reformulates agentic OPSD from a contrastive learning perspective. 
By leveraging predictive entropy to distinguish between positive positions (reflective exploration) and negative positions (exposure bias), CRPO conducts group-wise contrast to preserve reliable, fine-grained optimization signals. 
Extensive evaluations across 13 challenging reasoning and deep-search benchmarks demonstrate that CRPO consistently outperforms existing reinforcement learning and self-distillation baselines, significantly enhancing training stability and generalization in long-horizon interactions.
\end{abstract}

\section{Introduction}
Recent advances in post-training Large Language Models (LLM) gradually converge along two complementary paradigms. Among them,
Reinforcement Learning with Verifiable Rewards (RLVR)~\citep{dpsk-r1,ARPO} optimizes policies against verifiable outcome-level feedback, yielding accurate but sparse rewards, where a single scalar supervises an entire generation with coarse-grained credit assignment on intermediate steps.
To tackle this, On-Policy Distillation (OPD)~\citep{opdsurvey,rethinking-opd} and On-Policy Self-Distillation (OPSD)~\citep{OPSD,SDPO} supply dense, logit-level guidance from a stronger teacher model. We focus on OPSD, which provides a self-improved mechanism possessing both dense supervision and low computation cost. However, its guidance is inherently privileged, as the self-teacher model benefits from information advantage, and thus can be overconfident at positions where the student model is genuinely uncertain, leading to \emph{exposure bias}~\citep{RLSD}. This problem is further amplified in agentic post-training due to the multi-turn interactions with external environments, and we attribute it into two key phenomena: \textbf{[1] Self-teacher's reasoning route converges in interaction turns} and \textbf{[2] Optimization directions lose in multi-turn distillation}.

\begin{figure}[t]
    \centering
    \includegraphics[width=1\columnwidth]{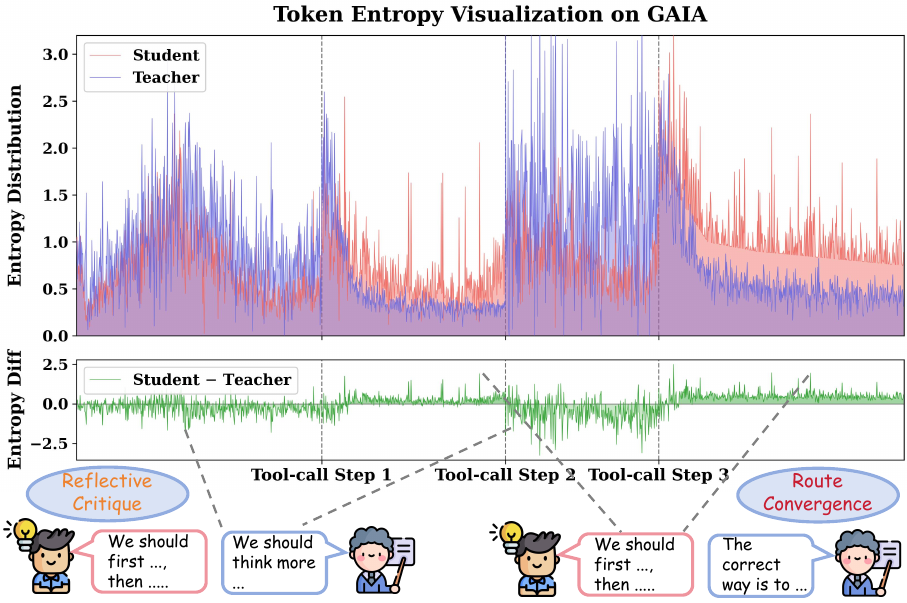}
    \caption{Position-level entropy along a student trajectory on a deep-search task in GAIA, compared with that of a privileged self-teacher prompted with a few demonstrations. A systematic statistical validation across all training rollouts is provided in Appendix~\ref{sec:entropy_collapse_stats}.}
    \label{fig: intro1}
\end{figure}

\textbf{[Phenomenon-1] Self-teacher's reasoning route converges in interaction turns.} Figure~\ref{fig: intro1} visualizes the position-level entropy along a student trajectory on a deep-search task, together with the entropy of a self-teacher obtained by prompting the same model with a few demonstrations.
After each tool call, the freshly introduced external information drives the student's entropy sharply upward. Under these circumstances with high uncertainty, it is observed that the privileged self-teacher \emph{may follow the reasoning routes of demonstrated cases to improve its confidence}, which leads to a significant entropy drop at these positions. This potentially hinders the model's generalization capability. 

\textbf{[Phenomenon-2] Optimization directions lose in multi-turn distillation.} Pure OPSD methods~\citep{OPSD,SDPO} utilize the KL-divergence to integrate the internalized reward signals from the privileged self-teacher, which is effective in single-turn tasks. However, in multi-turn long-horizon interactions, the reward signals are volunerable due to the introduction of external information. As above-mentioned, the teacher model is subjected to exposure bias, yielding to unreliable position-level supervision signals. As presented in Figure~\ref{fig: intro2}, compared with RLVR methods like GRPO~\citep{dpsk-r1} and ARPO~\citep{ARPO}, the reward curve of OPSD exhibits a fluctuating pattern while the KL-Divergence increases sharply, which indicates that the model is misled by the privileged self-teacher and fails to capture reward signals. 

\begin{figure}[t]
    \centering
    \includegraphics[width=1\columnwidth]{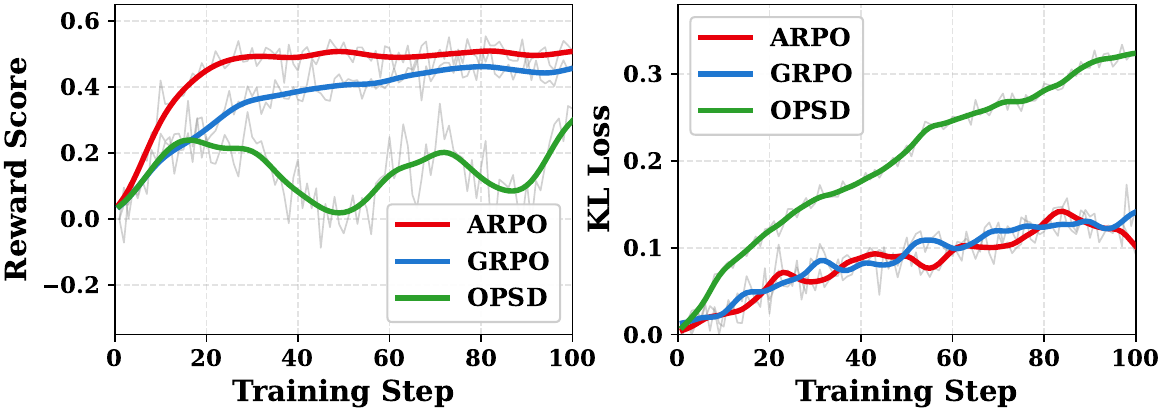}
    \caption{The Reward and KL-Divergence (token-average) of GRPO, ARPO, and OPSD during agentic post-training.}
    \label{fig: intro2}
\end{figure}

Based on the observations, we think the core problem is that the position-level supervision signals are unreliable in agentic post-training. To tackle above challenges, we think the pivots lie in: \textbf{[1] selectively distilling key positions and punishing ones with exposure bias}, and \textbf{[2] focusing on relatively advantageous positions to preserve clear optimization directions}. Though some existing studies~\citep{RLSD,SDAR,RLCSD} attempt to address the first problem, their solutions tend to rely on using self-distillation signals to control the update magnitude of RLVR, or to weightsum with RLVR's optimization objective. Since they take two separate frameworks into account, additional computational overhead is introduced. In contrast, we provide a solution to enhance the OPSD itself. 

To address the above-mentioned challenges, we introduce \textbf{CRPO}, a Contrastive Reinforced Policy Optimization method for agentic OPSD. In CRPO, we \emph{reformulate the OPSD framework from a contrastive learning perspective.} Specifically, we treat \ding{168} the original context and the context with privileged information as two views (i.e., the original view and the augmented view), \ding{169} the policy model as the encoder, \ding{170} the output logits as the representations, and \ding{171} the negative KL-divergence as the similarity function.
The objective is to perform distillation at positions that improve the model's generalization (positive pairs) while penalizing positions where exposure bias occurs (negative pairs).
To ensure clear optimization signals, we draw inspiration from the group-wise scheme of GRPO, conducting group-wise contrast across positions and encouraging relatively advantageous positions. Our contributions are summarized as:
\begin{itemize}[leftmargin=2em,itemsep=-0.1em]
\item[\ding{182}] We propose CRPO, an OPSD algorithm from a contrastive learning perspective, which enhances the training stability and performance in agentic post-training.
\item[\ding{183}] We use predictive-entropy to classify the positive and negative position and conduct group-wise contrast to preserve clear optimization directions. 
\item[\ding{184}] Beyond heurisitc motivation, the theoretical formula of CRPO is demonstrated equivalent to possess logit-wise credit assignment with a soft gate--see Appendix~\ref{sec:theoretical_analysis}.
\item[\ding{185}] Extensive evaluations on 13 challenging benchmarks show that CRPO consistently outperforms existing RL baselines across diverse reasoning and agentic tasks.
\end{itemize}

\section{Related Works}

\subsection{Agentic Reinforcement Learning}
Recently, Reinforcement Learning with Verifiable Rewards (RLVR)~\citep{lambert2025tulu3pushingfrontiers,dpsk-r1,DAPO} has achieved progress in enhancing reasoning tasks~\citep{we-math,dpsk-math,team2025qwq}.
Extending it to agentic settings, where models interact with external tools over multiple turns~\citep{ToolRL,ToolStar,CISPO}, introduces new challenges such as long-horizon credit assignment, high uncertainty after each tool call, and compounding errors across turns.
Recent efforts attempt to address these from several angles: adaptive exploration strategies~\citep{ARPO,AEPO} that exploit entropy spikes after tool interactions to encourage step-level sampling at high-uncertainty positions, and step-wise advantage attribution that internalizes the contribution of individual tool-use actions within the trajectory-level reward~\citep{otc,search-r1}.
Despite these advances, existing agentic RL methods rely solely on sparse outcome-level rewards and lack dense token-level signals, limiting sample efficiency and intermediate-step guidance.

\subsection{On-Policy Self-Distillation}
On-policy distillation (OPD)~\citep{opdsurvey,rethinking-opd,uni-opd,rubric-opd} provides dense, logit-level supervision by aligning the student's distribution with a stronger teacher at every position.
On-policy self-distillation (OPSD)~\citep{OPSD,SDPO,TCOD} simplifies this by letting a single model serve as both student and teacher, with the teacher granted privileged information to produce the dense signal at no extra model cost.
OPSD~\citep{OPSD} first introduces this self-distilled framework for single-turn reasoning.
SDPO~\citep{SDPO} frames self-distillation as an RL objective, using the teacher-student KL divergence as a dense reward signal.

Several works further explore hybrid methods combining OPSD with RLVR.
RLSD~\citep{RLSD} uses the self-distillation signal to modulate RLVR policy gradient updates.
SDAR~\citep{SDAR} jointly optimizes an RL objective and a distillation objective in multi-turn agentic settings. Skill-SD~\citep{Skill-SD} applies skill as the privileged information in agentic settings.
RLCSD~\citep{RLCSD} introduces contrastive signals to distinguish stylistic differences between successful and unsuccessful trajectories in reasoning tasks.
HDPO~\citep{HDPO} combines multiple supervision sources via privileged self-distillation into a unified policy optimization objective.
Our proposed CRPO, reformulates OPSD from a contrastive perspective~\citep{SimCLR,MoCo}, controling fine-grained distillation and preserving clear advantageous signals.

\begin{figure*}[!htbp]
    \centering
    \includegraphics[width=0.97\linewidth]{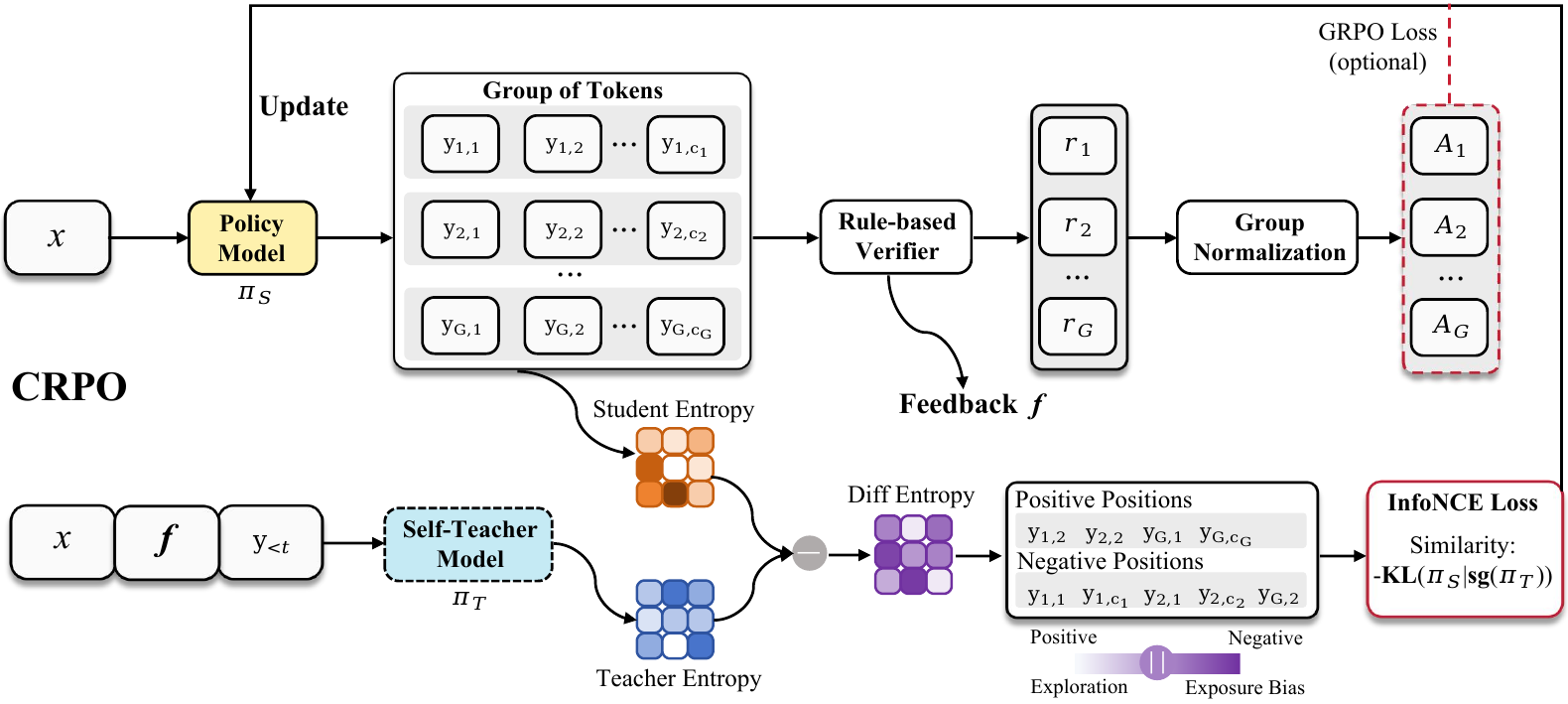}
    \caption{The algorithm pipeline of the proposed Contrastive Reinforced Policy Optimization (CRPO).}
    \label{fig:overview}
\end{figure*}

\section{Methodology} 

\subsection{Problem Definition}
We illustrate CRPO with a classical RLVR pipeline in Figure~\ref{fig:overview}. Given an initial task description $x$, we consider the agentic scenario where an LLM $\pi_\theta$ interacts with the environment through external tools, obtains observations from the tools, and then makes responses. Note that this process involves external tokens (e.g., searched information), which only serve as the context and do not provide optimization signals, so that we only formulate the valid response tokens for notational simplicity:
\begin{gather}
    y_i = (y_{i,1}, \cdots, y_{i,t}, \cdots, y_{i,c_i}) \sim \pi_\theta (\cdot | x),
\end{gather}
where $y_i$ denotes the sampled sequence and $c_i$ denotes the length. Total $\mathrm{G}$ rollouts are sampled per question $x$. In the context of OPSD, the student and teacher share the same policy model $\pi_\theta$. At each position $t$, the student model's context of the $i$-th sampled rollout is defined as:
\begin{gather}
    S_{i, t} = (x, y_{i, <t}),
\end{gather}
and the self-teacher model's context is formulated as:
\begin{gather}
    T_{i, t} = (x, f, y_{i, <t}),
\end{gather}
\begin{table}[b]
\centering
\small
\renewcommand{\arraystretch}{1.1}
\begin{tabular}{@{}p{1.2cm}p{6.3cm}@{}}
\toprule
\textbf{User:} & \texttt{<question>} \\
               & Reference solution: \texttt{<ref\_trajectory>} \\
               & Environmental info: \texttt{<env\_feedback>} \\
\midrule
\textbf{Assistant:} & \texttt{<student\_response>} \\
\bottomrule
\end{tabular}
\caption{Self-teacher prompt template.}
\label{tab:self-teacher-template}
\end{table}
where $f$ is the feedback, which is organized from the student's own rollouts, covering the reference successful trajectories from the same group, and the environmental feedbacks such as the compile information of code interpreter. The structure of self-teacher prompt is shown in Table~\ref{tab:self-teacher-template}. 

\subsection{Algorithm Framework}
We first revisit the classical OPSD~\citep{OPSD,SDPO} optimization objective, which adopts a reverse KL-divergence and stop the gradient flow from self-teacher to avoid privileged information leakage:
\begin{gather}
    \mathcal{L}_{\text{OPSD}}(\theta) = \frac{1}{\mathrm{G}} \sum_{i,t} \text{KL}(\pi_\theta(\cdot | S_{i,t})||\text{sg}(\pi_\theta(\cdot | T_{i,t}))),
\end{gather}
where sg($\cdot$) denotes the operation of stopping gradient. In original $\mathcal{L}_{\text{OPSD}}$, the student model consistently imitates the teacher’s behavior at all positions, which suffers from the exposure bias. In agentic settings, the self-teacher may also lead to reasoning route convergence, hindering the generalization capability. From our perspective, the intensity and direction of self-distillation at different positions both need to be adjusted to mitigate the issues mentioned above. 

We summarize some key observations:
\begin{tcolorbox}[colback=white,colframe=black,boxrule=0.8pt,arc=2pt,left=4pt,right=4pt,top=2pt,bottom=2pt,before skip=4pt,after skip=4pt]
\textbf{On-Policy Self-Distillation uses a single policy as the shared encoder, treats the original and privileged contexts as two views, and minimizes a distance function (KL-divergence) between their output distributions.}
\end{tcolorbox}
Inspired by classical contrastive learning~\citep{MoCo,SimCLR} frameworks, we observe that OPSD originally shares similarity to contrastive learning, but \emph{consistently treats the two views at each position as positive pairs.} Based on our motivation, not all positions are beneficial for distillation, and reasonably distinguishing positive and negative pairs may improve the robustness of optimization and enhance generalization.

Specifically, we reformulate the OPSD from a contrastive learning perspective in Figure~\ref{fig:contrast}. The core modification is to design a Judger to distinguish positive and negative positions in self-distillation, and it also utilizes the group-wise relative information. The details are presented in Figure~\ref{fig:overview}.

\begin{figure}[t]
    \centering
    \includegraphics[width=1\linewidth]{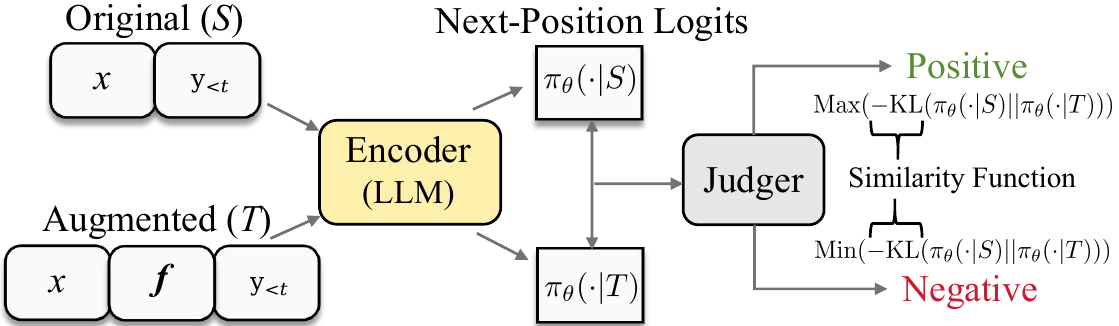}
    \caption{The CRPO pipeline, understanding from a contrastive learning perspective.}
    \label{fig:contrast}
\end{figure}

As illustrated in Figure~\ref{fig:contrast}, we rewrite CRPO in a constrative learning paradigm:
\begin{tcolorbox}[colback=white,colframe=black,boxrule=0.8pt,arc=2pt,left=4pt,right=4pt,top=2pt,bottom=2pt,before skip=4pt,after skip=4pt]
\ding{168} \textbf{Views:} student's context (original) and self-teacher's privileged context (augumented). \\
\ding{169} \textbf{Encoder:} the same LLM-based policy model. \\
\ding{170} \textbf{Representations:} the output logits at each position.\\
\ding{171} \textbf{Similarity Function:} the negative KL-divergence.
\end{tcolorbox}

To determine the positive and negative pairs, we utilize the predictive entropy~\citep{entropy} as a key metric, since it can reflect the position-level uncertainty of LLMs~\citep{ARPO,AEPO}, and help identify where potential exposure bias may be occurring. At position $t$ of the $i$-th rollout, we calculate the entropy difference $\Delta \mathrm{E}_{i,t}$ between the student and the self-teacher models:
\begin{gather}
    \mathrm{E}^{\text{S}}_{i,t} = \textstyle\sum_{\hat{y}_t} \pi_\theta(\hat{y}_t|S_{i,t}) \log \pi_\theta(\hat{y}_t|S_{i,t}),\\
    \mathrm{E}^{\text{T}}_{i,t} = \textstyle\sum_{\hat{y}_t} \pi_\theta(\hat{y}_t|T_{i,t}) \log \pi_\theta(\hat{y}_t|T_{i,t}),\\
    \Delta \mathrm{E}_{i,t} = \mathrm{E}^{\text{S}}_{i,t} - \mathrm{E}^{\text{T}}_{i,t},
\end{gather}
where the entropy difference $\Delta \mathrm{E}_{i,t}$ is treated as the key indicator to reflect the degree of fluctuation in uncertainty, which helps determine whether the self-teacher has merely followed the reference solutions contained in the privileged information (high $\Delta \mathrm{E}_{i,t}$), or whether it can properly leverage it to engage in reflective exploration (low $\Delta \mathrm{E}_{i,t}$). To make it more clear, we list the common cases in Table~\ref{tab:entropy-cases}:

\begin{table}[!htbp]
\centering
\renewcommand{\arraystretch}{1.2}
\begin{tabular}{@{}ll@{}}
\toprule
\textbf{Case} & \textbf{Interpretation} \\
\midrule
$\mathrm{E}^{\text{S}}_{i,t} \ll \mathrm{E}^{\text{T}}_{i,t}$ & Reflective critique (encouraging exploration) \\
$\mathrm{E}^{\text{S}}_{i,t} \approx \mathrm{E}^{\text{T}}_{i,t}$ & Borderline \\
$\mathrm{E}^{\text{S}}_{i,t} \gg \mathrm{E}^{\text{T}}_{i,t}$ & Exposure bias (route convergence) \\
\bottomrule
\end{tabular}
\caption{Common cases of entropy difference between the student and self-teacher.}
\label{tab:entropy-cases}
\end{table}
Despite the absolute numerical value of $\Delta \mathrm{E}_{i,t}$ can reflect the position advantages to some extent, we also introduce the group relative mechanism to further preserve the clear optimization directions from borderline cases. Specifically, we rank the entropy difference among all the positions of $\mathrm{G}$ rollouts, and define positive and negative pairs:
\begin{gather}
    \mathcal{P} = \{(i,t) \mid \Delta \mathrm{E}_{i,t} \leq \text{Bottom}_{p\%}(\{\Delta \mathrm{E}_{i,t}\}_{i=1, t=1}^{\mathrm{G}, c_i})\},\\
    \mathcal{N} = \{(i,t) \mid (i,t) \notin \mathcal{P}\},
\end{gather}
where $p$ denotes the proportion of positions selected as positive pairs, $\mathcal{P}$ is the positive set containing positions where the self-teacher engages in reflective exploration rather than simply copying the reference, and $\mathcal{N}$ is the negative set comprising the remaining positions that are more likely to suffer from exposure bias.

\subsection{Optimization Objective}
After defining the positive and negative pairs, we formulate the contrastive self-distillation objective. We adopt the negative KL-divergence as the similarity function between the student and the self-teacher at each position:
\begin{gather}
    \text{sim}(i,t) = -\text{KL}(\pi_\theta(\cdot | S_{i,t}) \| \text{sg}(\pi_\theta(\cdot | T_{i,t})))
\end{gather}
Following the InfoNCE~\citep{SimCLR,MoCo} framework, we place positive pairs in the numerator and all pairs in the denominator, which yields the contrastive self-distillation loss:
\begin{gather}
    \mathcal{L}_{\text{CRPO}}(\theta) = -\log \frac{\sum_{(i,t) \in \mathcal{P}} \exp(\text{sim}(i,t) / \tau)}{\sum_{(i,t) \in \mathcal{P} \cup \mathcal{N}} \exp(\text{sim}(i,t) / \tau)},
\end{gather}
where $\tau$ is the temperature hyperparameter. Minimizing $\mathcal{L}_{\text{CRPO}}$ reduces the KL divergence at positive positions and enlarges it at negative ones, which encouges exploration at pivotal positions, and mitigates blind imitation of exposure-biased predictions. The gradient of $\mathcal{L}_{\text{CRPO}}$ admits the form of a contrastively reweighted policy gradient--see Appendix~\ref{sec:crpo_gradient_proof} for the full derivation.
\begin{proposition}\label{prop:crpo_gradient}
The gradient of $\mathcal{L}_{\text{CRPO}}$ is
\begin{align}
    &\nabla_\theta \mathcal{L}_{\text{CRPO}}(\theta) = \mathbb{E}_{y_i \sim \pi_\theta(\cdot | x)} \bigg[ -\frac{1}{\tau} \sum_{i,t} c_{i,t} \notag \\
    &\cdot \mathbb{E}_{\hat{y}_t \sim \pi_\theta(\cdot | S_{i,t})} \Big[ \hat{A}_{i,t} \nabla_\theta \log \pi_\theta(\hat{y}_t | S_{i,t}) \Big] \bigg],
\end{align}
where $\hat{A}_{i,t} := \log \frac{\text{sg}(\pi_\theta(\hat{y}_t | T_{i,t}))}{\pi_\theta(\hat{y}_t | S_{i,t})}$ denotes the per-token advantage evaluated at a single sample $\hat{y}_t$, with $c_{i,t} \leq 0$ for positive pairs and $c_{i,t} \geq 0$ for negative pairs.
\end{proposition}
Proposition~\ref{prop:crpo_gradient} indicates that $\nabla_\theta \mathcal{L}_{\text{CRPO}}$ takes the form of a per-token policy gradient. Each position $(i,t)$ contributes a score-function term $\hat{A}_{i,t}\,\nabla_\theta \log \pi_\theta(\hat{y}_t \mid S_{i,t})$, in which the logit-level advantage $\hat{A}_{i,t}$ is supplied by the self-teacher rather than by a sparse trajectory-level reward.

The coefficient $c_{i,t}$ acts as a 
\emph{soft gate} on this gradient. Its sign determines the optimization direction: positive pairs drive the student toward the teacher, whereas negative pairs drive it away. Its magnitude is given by softmax-normalized similarities (see Appendix~\ref{sec:crpo_gradient_proof}), so that positions with more discriminative similarities receive larger weights within each rollout group.

The loss $\mathcal{L}_{\text{CRPO}}$ already provides a complete logit-level optimization signal and can be used as a standalone objective, which is the configuration adopted in our main experiments. It is also compatible with outcome-level RL objectives such as GRPO~\citep{dpsk-r1}, leading to a combined variant denoted as CRPO$^{*}$:
\begin{gather}
    \mathcal{L}_{\text{CRPO}^{*}}(\theta) = \mathcal{L}_{\text{GRPO}}(\theta) + \lambda \mathcal{L}_{\text{CRPO}}(\theta),
\end{gather}
\begin{table*}[t]
\centering
\setlength\tabcolsep{3.5pt}
\renewcommand{\arraystretch}{0.72}
\fontsize{9pt}{10.5pt}\selectfont
\begin{tabular}{p{3cm}ccccccccccc}
\toprule
\multirow{2}[2]{*}{\textbf{Method}} & \multicolumn{5}{c}{\textbf{Mathematical Reasoning}} & \multicolumn{5}{c}{\textbf{Knowledge-Intensive Reasoning}} & \multirow{2}[2]{*}{\textbf{Avg.}} \\
\cmidrule(lr){2-6} \cmidrule(lr){7-11}
& AIME24 & AIME25 & MATH500 & GSM8K & MATH & WebWalker & HQA & 2Wiki. & MuSiQ. & Bamb. & \\
\midrule
\textbf{Qwen2.5-3B-Instruct} & 10.0 & 6.7 & 63.0 & 75.0 & 71.6 & 0.5 & 9.7 & 9.4 & 3.6 & 11.7 & 26.1 \\
\quad + TIR Prompting                 & 6.7 & 6.7 & 52.2 & 56.6 & 62.8 & 14.0 & 15.4 & 14.1 & 6.1 & 16.4 & 25.1 \\
\quad + GRPO                 & 20.0 & 13.3 & \first{72.0} & \second{86.0} & 81.0 & 21.0 & 56.5 & 64.5 & 24.7 & 65.2 & 50.4 \\
\quad + ARPO                 & \second{23.3} & \first{20.0} & 71.4 & 85.0 & 82.5 & 24.5 & 58.5 & 67.4 & 28.7 & 66.8 & 52.8 \\
\quad + OPSD & 16.7 & 6.7 & 54.0 & 74.2 & 73.6 & 1.5 & 2.4 & 12.0 & 7.5 & 12.0 & 26.1  \\
\quad + SDPO                 & 16.7 & \second{16.7} & 71.2 & 85.0 & 81.0 & 17.8 & 52.4 & 62.5 & \second{30.0} & 63.6 & 49.7 \\
\quad + RLSD & 20.0 & 13.3 & 68.2 & 82.4 & 78.8 & 22.6 & 40.4 & 40.2 & 16.5 & \second{67.2} & 45.0 \\
\quad + \textbf{CRPO (Ours)} & \first{26.7} & \second{16.7} & 70.2 & 85.6 &  \second{83.2} & \second{28.0} & \first{60.4} & \first{70.2} & \second{30.0} & 66.8 & \second{53.8}\\
\quad + \textbf{CRPO* (Ours)} & \first{26.7} & \first{20.0} & \second{71.8} & \first{86.2} &  \first{83.6} & \first{28.6} & \second{59.8} & \second{68.6} & \first{33.3} & \first{70.4} & \first{54.9}\\

\midrule
\textbf{Llama3.1-8B-Instruct} & 3.3 & 0.0 & 43.3 & 81.4 & 60.6 & 3.0 & 24.3 & 24.6 & 10.4 & 40.0 & 29.1 \\
\quad + TIR Prompting        & 3.3 & 3.3 & 39.4 & 73.8 & 58.2 & 15.0 & 48.5 & 47.5 & 15.5 & 58.4 & 36.3 \\
\quad + GRPO                 & 13.3 & 13.3 & 62.4 & 87.4 & 79.2 & 26.5 & 57.8 & 71.8 & 31.0 & 68.2 & 51.1 \\
\quad + ARPO                 & \second{23.3} & 16.7 & 64.6 & 88.0 & \second{80.2} & 30.5 & \second{65.4} & 75.5 & 34.8 & 73.8 & 55.3 \\
\quad + OPSD & 13.3 & 10.0 & 52.6 & 77.6 & 62.4 & 15.4 & 3.2 & 18.8 & 12.4 & 16.2 & 28.2 \\
\quad + SDPO                 & 16.7 & 13.3 & 62.2 & 86.8 & 76.0 & 24.5 & 56.8 & 72.3 & 29.6 & 67.0 & 50.5  \\
\quad + RLSD &  20.0 & 13.3 &  66.4 & 78.2 & 79.2 & 22.4 & 46.8 & 50.2 & 22.6 & 73.0  & 47.2  \\
\quad + \textbf{CRPO (Ours)}  & \first{26.7} & \second{20.0} &  \second{68.0} & \second{88.2} & \second{80.2} & \second{31.8} & 64.8 & \second{78.2} & \second{35.0} & \first{75.6} & \second{56.9}     \\
\quad + \textbf{CRPO* (Ours)}  & \first{26.7} & \first{23.3} &  \first{71.4} & \first{88.6} & \first{80.8} & \first{35.2} & \first{66.8} & \first{81.2} & \first{35.2} & \second{75.4} & \first{58.5}     \\

\midrule
\textbf{Qwen2.5-7B-Instruct} & 10.0 & 10.0 & 70.6 & 90.2 & 82.0 & 2.0 & 12.2 & 12.6 & 6.6 & 24.0 & 32.0 \\
\quad + TIR Prompting        & 6.7 & 10.0 & 68.2 & 64.6 & 78.2 & 15.5 & 14.8 & 18.3 & 9.5 & 23.6 & 31.0 \\
\quad + GRPO                 & 23.3 & 26.7 & 78.0 & \first{92.8} & 87.8 & 22.0 & 59.0 & 76.1 & 30.6 & 68.4 & 56.5 \\
\quad + ARPO                 & \second{30.0} & \second{30.0} & 78.8 & 92.2 & \first{88.8} & 26.0 & 58.8 & 76.1 & \second{31.1} & 71.5 & 58.3 \\
\quad + OPSD & 20.0 & 16.7 & 72.8 & 74.8 & 84.2 & 14.6 & 2.5 & 20.0 & 14.5 & 22.4 & 34.3 \\
\quad + SDPO                 & 23.3 & 26.7 & \first{80.4} & 90.6 & \first{88.8} & 23.8 & 58.4 & 68.4 & 28.4 & 65.0 & 55.4 \\
\quad + RLSD & \second{30.0} & 26.7 & \second{79.0} & 91.2 & 86.8 & 24.6 & 45.5 & 48.9 & 21.5 & 73.0 & 52.7  \\
\quad + \textbf{CRPO (Ours)} & \second{30.0} & \first{33.3} & \first{80.4} & \second{92.4} & \second{88.4} & \second{30.0} & \second{59.8} & \second{77.4} & 30.8 & \second{73.5} & \second{59.6}  \\
\quad + \textbf{CRPO* (Ours)} & \first{33.3} & \first{33.3} & \first{80.4} & \first{92.8} & \first{88.8} & \first{32.0} & \first{62.6} & \first{78.2} & \first{31.5} & \first{74.0} & \first{60.7}  \\

\bottomrule
\end{tabular}
\caption{Performance on 10 challenging reasoning tasks. The best and second-best results are \textbf{bolded} and \underline{underlined}. Abbreviations: HQA (HotpotQA), 2Wiki. (2WikiMultiHopQA), MuSi. (MuSiQue), Bamb. (Bamboogle).}
\label{tab:main_table}
\end{table*}
where $\lambda$ controls the strength of the contrastive self-distillation. From a structural perspective, CRPO$^{*}$ preserves the trajectory-level reward optimization of GRPO and replaces only its regularizer: the constraint that anchors the policy to a frozen reference model is replaced by a position-aware contrastive KL coupling with the self-teacher through $\mathcal{L}_{\text{CRPO}}$. The outcome-level skeleton of GRPO is left unchanged, and the regularization target is upgraded from a static reference to the self-teacher. A component-wise comparison between CRPO$^{*}$ and classical GRPO is provided in Appendix~\ref{sec:crpo_star_vs_grpo}.

\section{Experiments}

\subsection{Experimental Setup}

\paragraph{Datasets.}
We evaluate CRPO on two families of long-horizon reasoning benchmarks. \textbf{(1) Mathematical \& Knowledge-Intensive Reasoning} includes the mathematical tasks AIME24, AIME25, MATH500~\citep{math500}, MATH~\citep{MATH}, and GSM8K, together with the knowledge-intensive tasks WebWalker~\citep{2501_WebWalker}, HotpotQA~\citep{hotpotqa}, 2WikiMultihopQA~\citep{2wiki}, MuSiQue~\citep{musique}, and Bamboogle~\citep{bamboogle}. \textbf{(2) Deep Search} includes GAIA~\citep{GAIA}, WebWalkerQA~\citep{2501_WebWalker}, Humanity's Last Exam (HLE)~\citep{HLE}, and the xbench-DeepSearch split~\citep{chen2025xbench,li2025websailor}. For consistency with prior work, we follow Tool-Star~\citep{dong2025toolstar} for the math and knowledge splits, and WebThinker~\citep{Webthinker} for the deep-search splits.

\paragraph{Baselines.}
We compare CRPO against three groups of baselines on the same backbones. The \textbf{direct-reasoning} group instantiates the instruct versions of Qwen2.5~\citep{qwen2.5} and Llama3.1~\citep{llama3}, with Qwen3~\citep{qwen3} used for deep search; we also report TIR prompting~\citep{search-o1} as a tool-use prompting baseline. The \textbf{trajectory-level RL} group includes GRPO~\citep{dpsk-r1} and ARPO~\citep{ARPO}. The \textbf{self-distillation} group includes OPSD~\citep{OPSD}, SDPO~\citep{SDPO}, and RLSD~\citep{RLSD}, which share the on-policy self-distillation formulation that CRPO extends.

\paragraph{Training Protocol.}
We adopt a cold-start SFT followed by RL pipeline~\citep{dong2025toolstar} to avoid reward collapse in the early RL phase. The cold-start phase uses the 54K Tool-Star corpus augmented with the 0.8K STILL set~\citep{cort}, fine-tuned with LLaMA-Factory~\citep{zheng2024llamafactory}. For the RL phase, we use the 10K Tool-Star samples for the math/knowledge tasks and 1K mixed hard samples drawn from SimpleDeepSearcher~\citep{SimpleDeepSearcher} and WebSailor~\citep{li2025websailor} for deep search. Tool interactions use top-10 Bing snippets and a sandboxed Python interpreter, with token-level F1 as the rule-based correctness signal.

\paragraph{Evaluation.}
For knowledge-intensive QA we report token-level F1; remaining tasks are scored via LLM-as-Judge with Qwen2.5-72B-Instruct. We use pass@1 with temperature $0.6$ and top-$p$ $0.95$, and extract answers from \texttt{\textbackslash box\{\}} following~\citet{search-o1}. Deep-search evaluation additionally enables a browser-equipped search engine.

\begin{table*}[t]
\centering
\setlength\tabcolsep{3.5pt}
\renewcommand{\arraystretch}{0.72}
\fontsize{9pt}{11.5pt}\selectfont
\begin{tabular}{
    p{2.8cm} % Method
    *{4}{>{\centering\arraybackslash}p{0.8cm}} % General AI Assistant
    *{4}{>{\centering\arraybackslash}p{0.8cm}}  % WebWalker (Easy, Med, Hard, Avg)
    *{4}{>{\centering\arraybackslash}p{0.75cm}} % Humanity's Last Exam
    >{\centering\arraybackslash}p{1.2cm}        % XBench
}
\toprule
\multirow{2}{*}{\textbf{Method}}
  & \multicolumn{4}{c}{\textbf{General AI Assistant}}
  & \multicolumn{4}{c}{\textbf{WebWalkerQA}}
  & \multicolumn{4}{c}{\textbf{Humanity's Last Exam}}
  & \textbf{XBench} \\
\cmidrule(lr){2-5} \cmidrule(lr){6-9} \cmidrule(lr){10-13} \cmidrule(lr){14-14}
  & Lv.1 & Lv.2 & Lv.3 & Avg.
  & Easy & Med. & Hard & Avg.
  & NS & CE & SF & Avg.
  & Avg. \\
\midrule

% 下面为正文数据部分，已按你的要求调整
\multicolumn{14}{l}{\textit{\textbf{Direct Reasoning (32B)}}} \\
Qwen3-32B-thinking & 26.2 & 12.1 & 0.0 & 15.5 & 6.9 & 1.1 & 2.9 & 3.1 & 14.6 & 9.8 & 8.4 & 12.6 & 14.0 \\
DeepSeek-R1-32B    & 21.5 & 13.6 & 0.0 & 14.2 & 7.5 & 1.4 & 4.2 & 3.8 & 6.6 & 5.1 & 6.5 & 6.4 & 10.0 \\
QwQ-32B            & 30.9 & 6.5  & 5.2 & 15.5 & 7.5 &2.1  &4.6  & 4.3 & 11.5 & 7.3 & 5.2 & 9.6 & 10.7 \\
\midrule
\multicolumn{14}{l}{\textit{\textbf{Single-Enhanced Method (Qwen3-8B)}}} \\
Search-o1    & 35.9 & 15.4 & 0.0  & 21.4 &  6.7 & 15.5 & 9.7 & 11.5 & 7.6 & 2.7 & 5.3  & 6.4 & 10.0 \\
WebThinker   & 43.6 & 11.5 & 0.0  & 22.3 & 6.7 & 13.1 & 16.9 & 13.0 & 7.3 & 4.0 & 6.3 & 6.6 & 13.0 \\
\multicolumn{14}{l}{\textit{\textbf{RL-based Method (Qwen3-8B)}}} \\
Qwen3-8B & 28.1 & 15.4 & \first{16.7} & 20.4 & 0.0 & 2.4 & 2.8 & 2.0 & 3.9 & 2.7 & 8.4  & 4.6 & 9.0 \\
\;\; + GRPO & 48.7 & 25.0 & \second{8.3} & 32.0 & 24.4 & \second{33.3} & 26.8 & 29.0 & 7.9 & 4.0 & 10.5  & 7.8 & 20.0\\
\;\; + ARPO & 53.9 & 32.7 & \first{16.7} & 38.8 & 26.7 & \second{33.3} & \second{29.6} & 30.5 & 7.3 & \second{6.7} & \first{15.8} & 8.8 & 25.0 \\
\;\; + OPSD & 28.1 & 17.3 & \second{8.3}  & 20.4 &  6.7 & 13.1 & 9.7 & 10.5 & 4.2 & 4.0 & 6.3 & 4.6 & 16.0 \\
\;\; + SDPO & 43.6 & 25.0 & \second{8.3} & 30.1 & 24.4 & 29.8 & 26.8 & 27.5 & 7.9 & \second{6.7} & 10.5 & 8.2 & 16.0\\
\;\; + RLSD & 51.3 & 30.8 & \first{16.7} & 36.9 & 31.1 & 29.8 & \second{29.6} & 30.0 & \second{10.9} & 5.3 & \second{13.7} & 10.6 & 17.0\\
\;\; + \textbf{CRPO (Ours)} & \second{61.5} & \second{38.5} & \first{16.7} & \second{44.7} & \second{35.6} & \first{39.3} & \first{36.6} & \second{37.5} & \second{10.9} & \second{6.7} & \second{13.7} & \second{10.8} & \second{26.0}\\
\;\; + \textbf{CRPO* (Ours)} & \first{64.1} & \first{42.3} & \first{16.7} & \first{47.6} & \first{37.8} & \first{39.3} & \first{36.6} & \first{38.0} & \first{13.0} & \first{8.0} & \first{15.8} & \first{12.8} & \first{27.0}\\

\midrule
\multicolumn{14}{l}{\textit{\textbf{Single-Enhanced Method (Qwen3-14B)}}} \\
Search-o1    & 48.7 & 23.1 & 0.0  & 30.1 &  11.1 & 21.4 & 16.9 & 17.5 & 6.4 & 4.0 & 10.5 & 6.8 & 21.0 \\
WebThinker   & 48.7 & 26.9 & \second{8.3}  & 33.0 &  13.3 & 23.8 & 18.3 & 19.5 & 7.0 & 4.0 & 9.5 & 7.0 & 23.0 \\

\multicolumn{14}{l}{\textit{\textbf{RL-based Method (Qwen3-14B)}}} \\
Qwen3-14B & 33.3 & 13.5 & 0.0 & 19.4 & 6.7 & 2.4 & 4.2 & 4.0 & 5.5 & 6.7 & 11.6 & 6.8 & 14.0 \\
\;\; + GRPO & 51.3 & 34.6 & 0.0 & 36.9 & 28.9 & 33.3 & 26.8 & 30.0 & 7.9 & 6.7 & 12.6 & 8.6 & 27.0 \\
\;\; + ARPO & 56.4 & \second{40.4} & \second{16.7} & 43.7 & 31.1 & 42.9 & 31.0 & 36.0 & 10.3 & 10.7 & 13.7 & 11.0 & 32.0 \\
\;\; + OPSD & 48.7 & 25.0 & 8.3  & 32.0 &  11.1 & 21.4 & 16.9 & 17.5 & 6.4  & 5.3 & 11.6 & 7.2 & 20.0 \\
\;\; + SDPO & 51.3 & 26.9 & \second{16.7} &  35.0 & 28.9 & 33.3 & 31.0 & 31.5  & 10.3 & 6.7 & 12.6 & 10.2 & 27.0 \\
\;\; + RLSD & 56.4 & 34.6 & \second{16.7} &  40.8 & 26.7 & 33.3 & 35.2 & 32.5 & 10.9 & 8.0 & 13.7 & 11.0 & 28.0 \\
\;\; + \textbf{CRPO (Ours)} & \second{64.1} & \first{44.2} & \second{16.7} & \second{48.5} & \first{42.2} & \second{51.2} & \second{39.4} & \second{45.0} & \second{13.6} & \second{13.3} & \second{15.8} & \second{14.0} & \second{34.0}\\
\;\; + \textbf{CRPO* (Ours)} & \first{66.7} & \first{44.2} & \first{25.0} & \first{50.5} & \second{40.0} & \first{52.4} & \first{40.9} & \first{45.5} & \first{13.9} & \first{14.7} & \first{16.8} & \first{14.6} & \first{36.0}\\
\bottomrule
\end{tabular}
\caption{Overall performance on various deep search tasks, with accuracy results for each dataset obtained using llm-as-judge. The best results are indicated in \textbf{bold}, and the second-best results are \underline{underlined}. }
\label{tab:main_result}
\end{table*}

\subsection{Main Results}
Tables~\ref{tab:main_table}--\ref{tab:main_result} report results on 13 reasoning and deep-search benchmarks. CRPO$^{*}$ ranks first on every backbone and every benchmark average. We highlight three observations.
\begin{itemize}[leftmargin=1em]
    \item[\ding{192}] \textbf{Vanilla self-distillation is fragile, while contrastive gating restores its effectiveness.} OPSD remains near the base-model level on both tables (e.g., $26.1\%$ and $34.3\%$ on Qwen2.5-3B and Qwen2.5-7B in Table~\ref{tab:main_table}, and $20.4\%$ GAIA-Avg.\ on Qwen3-8B in Table~\ref{tab:main_result}). Even reward-filtered variants such as SDPO and RLSD remain far below CRPO$^{*}$: on the four deep-search benchmarks, CRPO$^{*}$ exceeds the best self-distillation baseline by $7.7\%$ on average at the 8B scale (e.g., $+10.7\%$ on GAIA-Avg.\ and $+10.0\%$ on XBench) and by $8.6\%$ on average at the 14B scale (e.g., $+13.0\%$ on WebWalkerQA-Avg.). Comparable advantages of $6.9\%$ to $8.7\%$ are observed on the knowledge-intensive split of Table~\ref{tab:main_table}. These results confirm that trajectory-level selection without a token-level signal is insufficient at high-uncertainty tool-call boundaries.
    \item[\ding{193}] \textbf{Position-aware contrast scales with horizon length.} Without any outcome reward, CRPO already surpasses ARPO across all backbones, and the gap widens on long-horizon tasks. Specifically, CRPO$^{*}$ improves over ARPO by $3.2\%$ on average for Llama3.1-8B in Table~\ref{tab:main_table}, and by $8.8\%$ and $7.5\%$ on GAIA and WebWalkerQA for Qwen3-8B, and by $6.8\%$ and $9.5\%$ for Qwen3-14B in Table~\ref{tab:main_result}. The largest gains are observed precisely where reward-only signals become sparsest.
    \item[\ding{194}] \textbf{Orthogonality to outcome-level RL and competitiveness with larger models.} Adding the contrastive regularizer on top of GRPO yields a consistent improvement of $1.1\%$ to $1.6\%$ on the average score of every backbone in Table~\ref{tab:main_table}, and further gains of up to $2.9\%$ on individual deep-search benchmarks in Table~\ref{tab:main_result}. Moreover, the resulting $8$B and $14$B CRPO$^{*}$ models exceed $32$B direct-reasoning baselines (Qwen3-32B-thinking, QwQ-32B, DeepSeek-R1-32B) by more than $30$ absolute points on GAIA-Avg., demonstrating that CRPO$^{*}$ delivers gains in agentic reasoning that cannot be matched by scaling parameters alone.
\end{itemize}

\subsection{Analyzing Sampling at Scale}
\label{sec:passk}
\begin{figure*}[t]
    \centering
    \includegraphics[width=\linewidth]{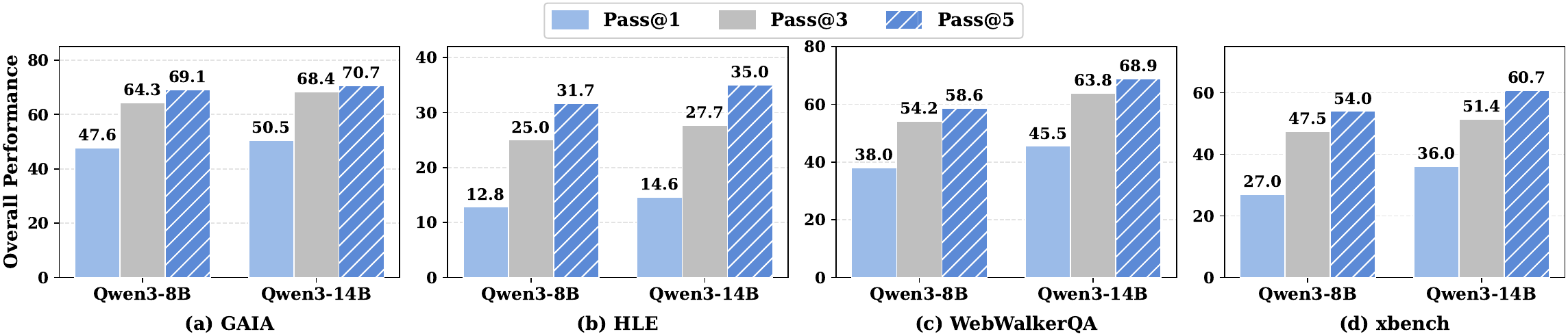}
    \caption{Pass@1, Pass@3, and Pass@5 of CRPO$^{*}$ on Qwen3-8B and Qwen3-14B across four deep-search benchmarks (GAIA, HLE, WebWalkerQA, xbench).}
    \label{fig:passk}
\end{figure*}
Due to the dynamic and multi-round nature of deep-search evaluation, we further evaluate CRPO$^{*}$ with Pass@3 and Pass@5 on Qwen3-8B and Qwen3-14B, as illustrated in Figure~\ref{fig:passk}. Both backbones exhibit a consistent and substantial scaling trend from Pass@1 to Pass@5 across all four benchmarks. Notably, Qwen3-14B with CRPO$^{*}$ reaches \textbf{$70.7\%$ on GAIA, $35.0\%$ on HLE, $68.9\%$ on WebWalkerQA, and $60.7\%$ on xbench} at Pass@5, exceeding its Pass@1 counterparts by $20.2\%$, $20.4\%$, $23.4\%$, and $24.7\%$ respectively. We attribute this stable improvement to the position-aware contrastive self-distillation objective in CRPO$^{*}$, which encourages diverse and informative tool-use behaviors and thereby expands the effective sampling space, jointly achieving sampling diversity and inference effectiveness.

\subsection{Training Dynamics}
\label{sec:training_dynamics}
Figure~\ref{fig:training_dynamics} reports four metrics during the post-training. We observe that CRPO and CRPO$^{*}$ keep both entropy and KL within a narrow band that closely tracks ARPO and stays well below GRPO, due to their selective distillation at key positions, while their accuracy curves consistently dominate all baselines training; in contrast, OPSD exhibits a clear entropy collapse between steps~$60$ and~$100$, where the student model suffers from exposure bias and cannot achieve high-quality reasoning without privileged information, and the KL values also reflect that the student model deviates significantly from the original model during training, this also explains the sub-optimal final score in Table~\ref{tab:main_result}. Meanwhile, the total tool calls of CRPO and CRPO$^{*}$ stabilize only slightly above ARPO and far below GRPO and OPSD, indicating that the position-aware contrastive objective preserves tool-use efficiency without explicit budget constraints.
\begin{figure}[!htbp]
    \centering
    \includegraphics[width=\linewidth]{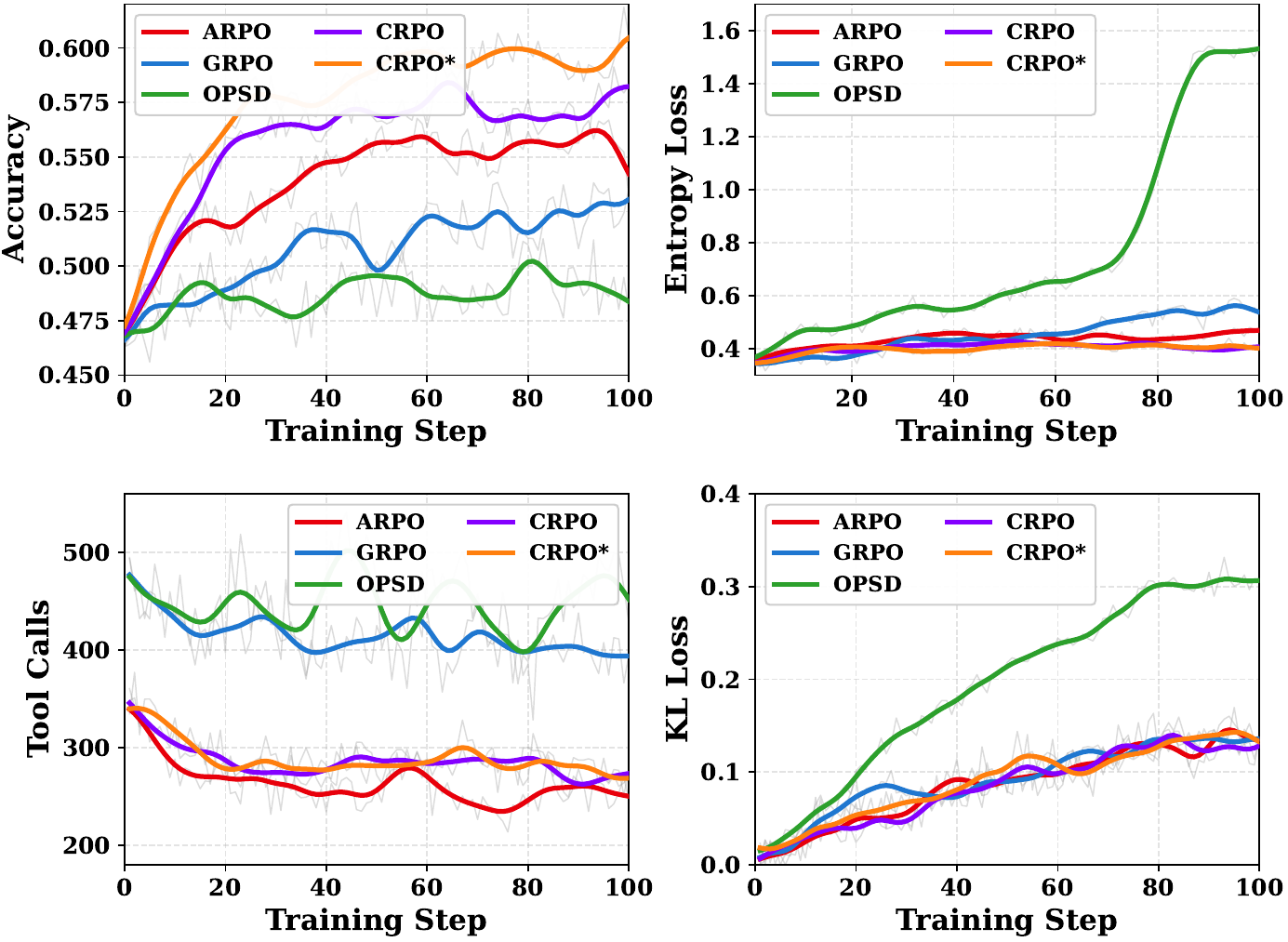}
    \caption{Training dynamics of CRPO and CRPO$^{*}$ against ARPO, GRPO, and OPSD on Qwen3-8B across $100$ optimization steps. From left to right and top to bottom: Accuracy, Entropy Loss, Tool Calls, and KL Loss (token-average).}
    \label{fig:training_dynamics}
    \vspace{-2mm}
\end{figure}

\subsection{Hyper-parameter Analysis}
\label{sec:hparams}
\begin{figure}[!htbp]
    \centering
    \includegraphics[width=\linewidth]{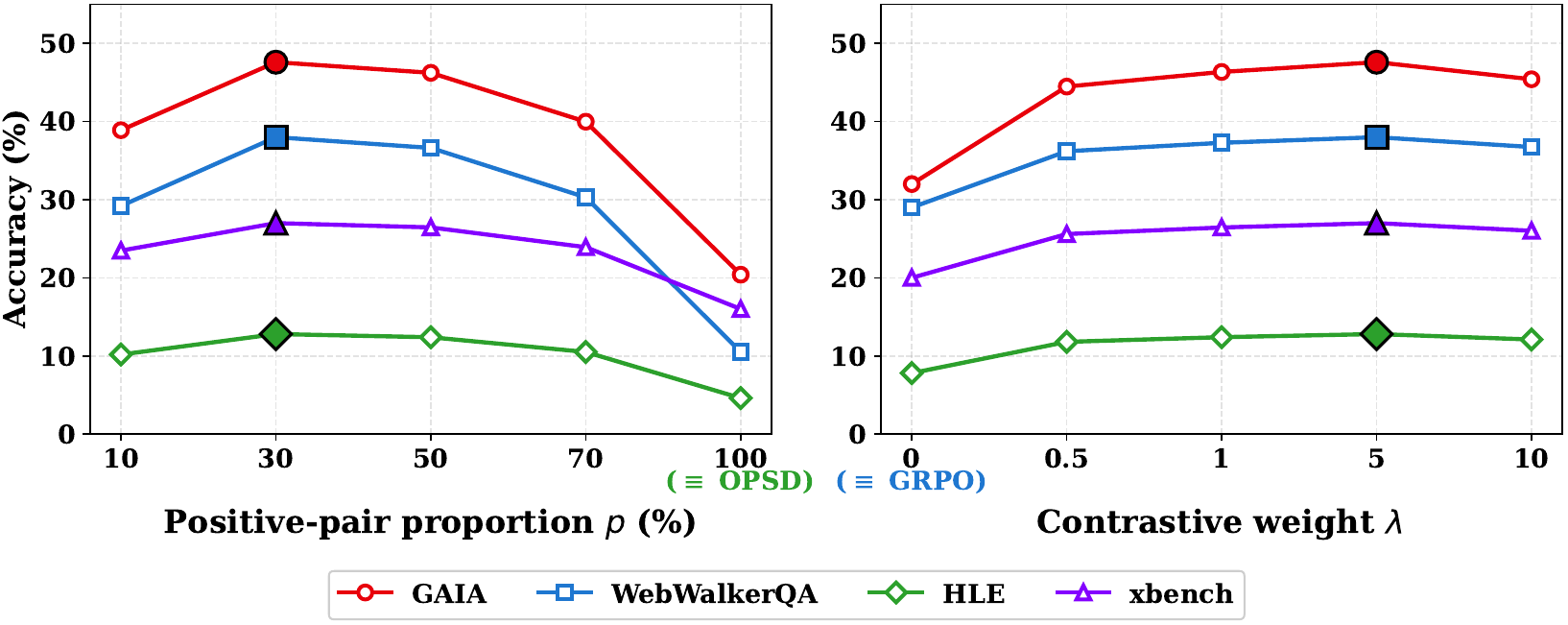}
    \caption{Hyper-parameter sweeps of CRPO$^{*}$ on Qwen3-8B across the four deep-search benchmarks. \textbf{Left:} positive-pair proportion $p$; $p\!=\!100$ degenerates to OPSD. \textbf{Right:} contrastive weight $\lambda$; $\lambda\!=\!0$ degenerates to GRPO.}
    \label{fig:hparams}
\end{figure}
We further study the sensitivity of CRPO$^{*}$ to its two key hyper-parameters: the positive-pair proportion $p$ that controls how many positions are selected as positive samples in the contrastive objective, and the regularization weight $\lambda$ in $\mathcal{L}_{\text{CRPO}^{*}} = \mathcal{L}_{\text{GRPO}} + \lambda\,\mathcal{L}_{\text{CRPO}}$. Figure~\ref{fig:hparams} reports the deep-search accuracy on Qwen3-8B as each hyper-parameter is varied. Both curves consistently peak at our default setting ($p\!=\!30$, $\lambda\!=\!5$) across all four benchmarks. When $p\to 100$, all positions are uniformly treated as positive pairs and CRPO$^{*}$ collapses to OPSD, suffering a substantial drop on every benchmark. Conversely, when $\lambda\to 0$, the contrastive regularizer vanishes and CRPO$^{*}$ degenerates to GRPO, again incurring a degradation. These two boundary cases jointly confirm that both the position-aware positive--negative gating and the contrastive coupling with the self-teacher are necessary; setting either of them to its degenerate value collapses CRPO$^{*}$ to a strictly weaker baseline. Additional ablations on the necessity of negative-pair repulsion and sensitivity to global rollout size are provided in Appendix~\ref{sec:ablation_negative} and~\ref{sec:rollout_sensitivity}.

\section{Conclusion}
In this paper, we introduce Contrastive Reinforced Policy Optimization (CRPO) to overcome unreliable position-level supervision in agentic On-Policy Self-Distillation. 
By reformulating OPSD as a contrastive objective, CRPO uses predictive entropy to selectively gate distillation positions, driving reflective exploration while mitigating exposure bias. 
Operating as a contrastively reweighted token-level policy gradient, CRPO seamlessly integrates with outcome-level RL methods like GRPO. 
Extensive evaluations across 13 benchmarks confirm that CRPO achieves state-of-the-art performance and robustly scales on long-horizon agentic tasks.

\bibliography{aaai2027}

\clearpage
\appendix
\makeatletter
\def\addcontentsline#1#2#3{%
  \addtocontents{#1}{\protect\contentsline{#2}{#3}{\thepage}{}}}
\makeatother
\renewcommand{\contentsname}{Table of Contents}
\setcounter{tocdepth}{2}
\tableofcontents

\section{Datasets and Baselines}\label{sec:datasets_and_baselines}

Below we detail the benchmarks and competing methods used throughout our experiments.

\subsection{Datasets}\label{sec:appendix_datasets}

Our evaluation covers three benchmark families that probe long-horizon reasoning at increasing levels of environmental interaction. We adopt the Tool-Star~\citep{dong2025toolstar} test split for the mathematical and knowledge-intensive tasks, and the WebThinker~\citep{Webthinker} split for the deep-search tasks.

\subsubsection{Mathematical Reasoning Benchmarks}
\begin{itemize}[leftmargin=1em]
\item \textbf{AIME24} contains 30 competition-grade problems from the 2024 American Invitational Mathematics Examination, spanning algebra, geometry, number theory, and combinatorics. Its difficulty makes it a discriminative testbed for frontier reasoning models.

\item \textbf{AIME25} comprises 30 problems from both AIME~I and AIME~II of February 2025, with comparable topical coverage to AIME24 but newer problem instances that reduce memorization effects.

\item \textbf{MATH500}~\citep{math500} is a curated 500-problem subset of the MATH corpus, covering college-level algebra, geometry, calculus, and number theory, and serving as a standard reasoning evaluation set.

\item \textbf{MATH}~\citep{MATH} is the full-scale counterpart, offering thousands of problems across multiple difficulty tiers and mathematical subfields, enabling fine-grained analysis of reasoning progression.

\item \textbf{GSM8K} targets elementary arithmetic word problems requiring 2--8 reasoning steps, primarily testing multi-step numerical reasoning rather than advanced mathematical knowledge.
\end{itemize}

\subsubsection{Knowledge-Intensive Reasoning Benchmarks}
\begin{itemize}[leftmargin=1em]
\item \textbf{HotpotQA}~\citep{hotpotqa} requires models to retrieve and compose evidence from multiple Wikipedia paragraphs, testing multi-hop reasoning over open-domain factual knowledge.

\item \textbf{2WikiMultihopQA}~\citep{2wiki} further stresses cross-document reasoning by constructing questions that explicitly require bridging information from two distinct Wikipedia articles.

\item \textbf{MuSiQue}~\citep{musique} raises the compositional depth beyond two hops, demanding that models chain evidence across several documents while resisting single-hop shortcut answers.

\item \textbf{Bamboogle}~\citep{bamboogle} consists of manually crafted multi-hop questions specifically designed so that no single retrieval query suffices, probing genuine compositional reasoning under retrieval-augmented settings.

\item \textbf{WebWalker}~\citep{2501_WebWalker} presents 680 QA pairs that require navigating live web-page structures, testing a model's ability to follow hyperlinks and reason over dynamically rendered content.
\end{itemize}

\subsubsection{Deep Search Benchmarks}
\begin{itemize}[leftmargin=1em]
\item \textbf{GAIA}~\citep{GAIA} offers 466 questions that jointly demand reasoning, web browsing, and tool invocation, functioning as an end-to-end evaluation of general-purpose AI assistants across three difficulty levels.

\item \textbf{Humanity's Last Exam (HLE)}~\citep{HLE} targets the frontier of model capability with interdisciplinary problems requiring abstract thinking and deep domain expertise, intended as an upper-bound probe for reasoning intelligence.

\item \textbf{WebWalkerQA}~\citep{2501_WebWalker} builds upon WebWalker with longer interaction horizons, requiring multi-step page traversal and cross-page information synthesis.

\item \textbf{xbench-DeepSearch}~\citep{chen2025xbench} evaluates breadth-and-depth search jointly, probing whether agentic models can explore a wide search space while maintaining coherent multi-step reasoning. We adopt the WebSailor~\citep{li2025websailor} evaluation protocol.
\end{itemize}

\subsection{Baselines}\label{sec:appendix_baselines}

We organize baselines into four groups---\textit{direct reasoning}, \textit{trajectory-level RL}, \textit{on-policy self-distillation}, and \textit{LLM-based search agents}---all evaluated on the same backbone checkpoints as CRPO.

\subsubsection{Direct Reasoning}
\begin{itemize}[leftmargin=1em]
\item \textbf{Qwen2.5 Series}~\citep{qwen2.5} provides general-purpose, code-oriented, and math-specialized variants pretrained on large multilingual corpora. We use the instruct versions at 3B and 7B scales as backbones for the mathematical and knowledge-intensive tracks.

\item \textbf{Llama3.1 Series}~\citep{llama3} offers dense models at 8B, 70B, and 405B parameters with extended context support. We adopt the 8B-Instruct variant as a cross-family backbone to test generalization.

\item \textbf{Qwen3 Series}~\citep{qwen3} spans 0.6B--235B parameters across dense and MoE configurations, with native dual-mode inference (thinking for complex reasoning, non-thinking for simple queries). We use Qwen3-8B and Qwen3-14B for the deep-search track.

\item \textbf{QwQ}~\citep{team2025qwq} is a 32B dense reasoning model trained with staged RL, targeting mathematical and logical tasks. We include it as a strong 32B direct-reasoning reference.

\item \textbf{DeepSeek-R1}~\citep{dpsk-r1} produces extended chains of thought through RL-trained reflection and self-verification. Its 32B distilled variant serves as another 32B-scale reference.

\item \textbf{TIR Prompting}~\citep{search-o1} augments a frozen base model with structured tool-invocation prompts at inference time, representing a parameter-free tool-use baseline.
\end{itemize}

\subsubsection{Trajectory-level RL Algorithms}
\begin{itemize}[leftmargin=1em]
\item \textbf{GRPO}~\citep{dpsk-r1} computes group-relative advantages across sampled rollouts and optimizes a clipped surrogate objective with a KL anchor to a frozen reference. It provides the outcome-level RL skeleton that CRPO$^{*}$ builds upon.

\item \textbf{ARPO}~\citep{ARPO} adapts GRPO to multi-turn agentic settings by introducing entropy-adaptive rollout allocation---sampling more branches at high-uncertainty steps following tool calls---and a step-level advantage attribution, representing the state-of-the-art trajectory-level method for agentic RL.
\end{itemize}

\subsubsection{On-Policy Self-Distillation}
\begin{itemize}[leftmargin=1em]
\item \textbf{OPSD}~\citep{OPSD} conditions the same policy on privileged rollout feedback to form a self-teacher, then uniformly minimizes the student--teacher KL across all positions. CRPO directly generalizes this formulation.

\item \textbf{SDPO}~\citep{SDPO} augments OPSD by filtering out unsuccessful rollouts before distillation, reducing noise from low-reward trajectories while still applying uniform per-position updates.

\item \textbf{RLSD}~\citep{RLSD} hybridizes GRPO advantages with the teacher--student log-probability gap to produce token-level reweighting, combining trajectory rewards with self-distillation signals. It is the strongest prior self-distillation method.
\end{itemize}

\subsubsection{LLM-based Search Agents}
\begin{itemize}[leftmargin=1em]
\item \textbf{Search-o1}~\citep{search-o1} couples agentic retrieval-augmented generation with a document-level reasoning module, enabling the model to iteratively close knowledge gaps during multi-step inference.

\item \textbf{WebThinker}~\citep{Webthinker} trains large reasoning models to autonomously browse, extract, and synthesize web content through DPO on iteratively generated tool-use trajectories, representing the current open-source deep-research baseline.
\end{itemize}

\section{Implementation Details}\label{sec:appendix_implementation}

In this section, we report the implementation details of CRPO and CRPO$^{*}$. To ensure a fair comparison with the strongest agentic RL baseline, our overall training pipeline and the shared hyper-parameters largely follow ARPO~\citep{ARPO}; we additionally describe the CRPO-specific hyper-parameters introduced by the contrastive self-distillation objective.

\subsection{Supervised Fine-Tuning}

Following the cold-start SFT then RL paradigm~\citep{r1searcher,dong2025toolstar}, we first fine-tune the backbone with supervised data using the LLaMA-Factory~\citep{zheng2024llamafactory} framework. We use a learning rate of $7 \times 10^{-6}$, with DeepSpeed ZeRO-3 and FlashAttention-2 for memory optimization. The batch size is set to 128, weight decay to 0.1, and the model is trained for 3 epochs with BF16 mixed precision and a maximum input length of 4096 tokens. We use Tool-Star's open-source 54K samples augmented with the 0.8K STILL set drawn from CORT~\citep{cort} for richer mathematical reasoning supervision.

\subsection{Reinforcement Learning}

We implement CRPO and CRPO$^{*}$ on top of the VERL framework. Tool invocation results (search snippets, Python execution outputs) are masked from the loss to prevent the policy from being biased toward tool outputs; the loss is computed only over tokens generated by the model itself, including textual reasoning and tool requests. We differentiate the settings for the two task families.

\paragraph{Deep Reasoning Tasks.}
For 7B-scale backbones, on both CRPO/CRPO$^{*}$ and the trajectory-level RL baselines, we adopt a total training batch size of 128, a PPO mini-batch size of 16, a global rollout size of 16, and an initial sampling size of 8. The maximum response length per interaction is capped at 4096 tokens. To stabilize training, the KL coefficient against the reference model is set to 0 (the contrastive regularizer in CRPO$^{*}$ already provides a position-aware anchor). The RL phase runs for 2 epochs.

\paragraph{Deep Search Tasks.}
For 8B-scale backbones, we keep the same configuration as in deep reasoning except that the maximum response length is extended to 8192 tokens to accommodate long-horizon web traversal. For 14B-scale backbones, we use the same hyper-parameters but distribute training across 16 NVIDIA H800 GPUs. Since the deep-search RL set contains only 1K mixed hard samples drawn from SimpleDeepSearcher~\citep{SimpleDeepSearcher} and WebSailor~\citep{li2025websailor}, the RL phase runs for 5 epochs.

\paragraph{CRPO-Specific Hyper-Parameters.}
On top of the shared training configuration, CRPO introduces two algorithm-specific hyper-parameters: the positive-pair proportion $p$ used by the entropy-gap-based judger and the contrastive weight $\lambda$ in $\mathcal{L}_{\text{CRPO}^{*}} = \mathcal{L}_{\text{GRPO}} + \lambda\,\mathcal{L}_{\text{CRPO}}$. Following the hyper-parameter analysis in Section~\ref{sec:hparams}, we set $p = 30\%$ and $\lambda = 5$ as the default values used in all main experiments. The contrastive temperature $\tau$ is fixed to $1.0$, and the top-$K$ vocabulary truncation in Section~\ref{sec:topk_approximation} is set to $K = 100$ in all runs. For CRPO$^{*}$, the EMA self-teacher coefficient is set to $\alpha = 0.1$.

\subsection{Search and Browsing Setup}

During both training and evaluation we use the Bing Web Search API as the retriever, configured with the US-English (US-EN) locale, and retrieve 10 web pages per query. For mathematical and knowledge-intensive reasoning, only the top-10 snippets are used as evidence. For deep-search tasks, each retrieved page is fetched up to 6000 tokens, and a browser agent of the same size as the reasoning backbone is used to refine and summarize the retrieved content.

\subsection{Hardware.}
All RL experiments are conducted on 16 NVIDIA H800 GPUs. We use BF16 mixed precision throughout. Cold-start SFT runs share the same hardware. The total wall-clock cost of CRPO$^{*}$ is comparable to that of ARPO under the same global rollout budget, since the contrastive regularizer reuses the same forward passes already required for the entropy-gap judger.

\section{Theoretical Analysis}\label{sec:theoretical_analysis}

\subsection{Proof of Proposition~\ref{prop:crpo_gradient}}\label{sec:crpo_gradient_proof}

We restate the proposition for convenience.

\begin{proposition}[Gradient of $\mathcal{L}_{\text{CRPO}}$]
The gradient of the CRPO loss is:
\begin{align}
    &\nabla_\theta \mathcal{L}_{\text{CRPO}}(\theta) = \mathbb{E}_{y_i \sim \pi_\theta(\cdot | x)} \bigg[ -\frac{1}{\tau} \sum_{i,t} c_{i,t} \notag \\
    &\cdot \mathbb{E}_{\hat{y}_t \sim \pi_\theta(\cdot | S_{i,t})} \Big[ \hat{A}_{i,t} \nabla_\theta \log \pi_\theta(\hat{y}_t | S_{i,t}) \Big] \bigg],
\end{align}
where $\hat{A}_{i,t} := \log \frac{\text{sg}(\pi_\theta(\hat{y}_t | T_{i,t}))}{\pi_\theta(\hat{y}_t | S_{i,t})}$ is the per-token advantage evaluated at a single sample $\hat{y}_t \sim \pi_\theta(\cdot | S_{i,t})$, and the contrastive weights are:
\begin{equation}
    c_{i,t} = \begin{cases}
        w_{i,t}^{\text{all}} - w_{i,t}^{\mathcal{P}}, & (i,t) \in \mathcal{P}, \\
        w_{i,t}^{\text{all}}, & (i,t) \in \mathcal{N},
    \end{cases}
\end{equation}
with:
\begin{align}
    w_{i,t}^{\text{all}} &= \frac{\exp(\text{sim}(i,t) / \tau)}{\sum_{(j,k) \in \mathcal{P} \cup \mathcal{N}} \exp(\text{sim}(j,k) / \tau)}, \\
    w_{i,t}^{\mathcal{P}} &= \frac{\exp(\text{sim}(i,t) / \tau)}{\sum_{(j,k) \in \mathcal{P}} \exp(\text{sim}(j,k) / \tau)}.
\end{align}
Moreover, $c_{i,t} \leq 0$ for $(i,t) \in \mathcal{P}$ and $c_{i,t} \geq 0$ for $(i,t) \in \mathcal{N}$.
\end{proposition}

\begin{proof}
We proceed in two steps.

\paragraph{Step 1: Differentiating the contrastive loss.}

Recall:
\begin{align}
    \mathcal{L}_{\text{CRPO}} = -\log \frac{\sum_{(i,t) \in \mathcal{P}} \exp(\text{sim}(i,t) / \tau)}{\sum_{(i,t) \in \mathcal{P} \cup \mathcal{N}} \exp(\text{sim}(i,t) / \tau)}.
\end{align}
Rewriting as a difference of log-sum-exp:
\begin{align}
    \mathcal{L}_{\text{CRPO}} &= \underbrace{\log \!\sum_{(i,t) \in \mathcal{P} \cup \mathcal{N}} \!\exp\!\big(\tfrac{\text{sim}(i,t)}{\tau}\big)}_{:= L_1} \notag \\
    &\quad - \underbrace{\log \sum_{(i,t) \in \mathcal{P}} \exp\!\big(\tfrac{\text{sim}(i,t)}{\tau}\big)}_{:= L_2}.
\end{align}
By the chain rule on $\log\sum\exp$:
\begin{align}
    \nabla_\theta L_1 &= \frac{1}{\tau} \!\!\sum_{(i,t) \in \mathcal{P} \cup \mathcal{N}} \!\!w_{i,t}^{\text{all}} \cdot \nabla_\theta \text{sim}(i,t), \\
    \nabla_\theta L_2 &= \frac{1}{\tau} \sum_{(i,t) \in \mathcal{P}} w_{i,t}^{\mathcal{P}} \cdot \nabla_\theta \text{sim}(i,t).
\end{align}
Combining:
\begin{align}
    &\nabla_\theta \mathcal{L}_{\text{CRPO}} = \nabla_\theta L_1 - \nabla_\theta L_2 \notag \\
    &= \frac{1}{\tau} \sum_{i,t} c_{i,t} \cdot \nabla_\theta \text{sim}(i,t). \label{eq:crpo_step1}
\end{align}

\textbf{Sign analysis.} For $(i,t) \in \mathcal{P}$: the denominator of $w_{i,t}^{\text{all}}$ is no smaller than that of $w_{i,t}^{\mathcal{P}}$ (same numerator, superset in denominator), so $w_{i,t}^{\text{all}} \leq w_{i,t}^{\mathcal{P}}$ and $c_{i,t} \leq 0$. For $(i,t) \in \mathcal{N}$: $c_{i,t} = w_{i,t}^{\text{all}} \geq 0$.

\paragraph{Step 2: Per-position similarity gradient.}

We derive $\nabla_\theta \text{sim}(i,t)$ following the RL policy gradient derivation. Since $\text{sim}(i,t) = -\text{KL}(\pi_\theta(\cdot | S_{i,t}) \| \text{sg}(\pi_\theta(\cdot | T_{i,t})))$, we compute:
\begin{align*}
    &\nabla_\theta \text{KL}(\pi_\theta(\cdot | S_{i,t}) \| \text{sg}(\pi_\theta(\cdot | T_{i,t}))) \\
    &= \nabla_\theta \sum_{\hat{y}_t} \pi_\theta(\hat{y}_t | S_{i,t}) \log \frac{\pi_\theta(\hat{y}_t | S_{i,t})}{\text{sg}(\pi_\theta(\hat{y}_t | T_{i,t}))}.
\end{align*}
Let $\hat{A}_{i,t} := \log \frac{\text{sg}(\pi_\theta(\hat{y}_t | T_{i,t}))}{\pi_\theta(\hat{y}_t | S_{i,t})}$, where the hat indicates dependence on a single sample $\hat{y}_t$. Then:
\begin{align*}
    &= -\nabla_\theta \sum_{\hat{y}_t} \pi_\theta(\hat{y}_t | S_{i,t}) \cdot \hat{A}_{i,t} \\
    &= -\sum_{\hat{y}_t} \Big[ \pi_\theta(\hat{y}_t | S_{i,t}) \nabla_\theta \hat{A}_{i,t} + \hat{A}_{i,t} \nabla_\theta \pi_\theta(\hat{y}_t | S_{i,t}) \Big].
\end{align*}
We have $\nabla_\theta \hat{A}_{i,t} = -\nabla_\theta \log \pi_\theta(\hat{y}_t | S_{i,t})$ (the negative score function, since the teacher is under stop-gradient). For the first term, using the score trick $\pi_\theta(\hat{y}_t | S_{i,t}) \nabla_\theta \log \pi_\theta(\hat{y}_t | S_{i,t}) = \nabla_\theta \pi_\theta(\hat{y}_t | S_{i,t})$:
\begin{align*}
    &-\sum_{\hat{y}_t} \pi_\theta(\hat{y}_t | S_{i,t}) \nabla_\theta \hat{A}_{i,t} \\
    &= \sum_{\hat{y}_t} \nabla_\theta \pi_\theta(\hat{y}_t | S_{i,t}) \\
    &= \nabla_\theta \underbrace{\sum_{\hat{y}_t} \pi_\theta(\hat{y}_t | S_{i,t})}_{=1} = 0.
\end{align*}
Thus the gradient of the KL reduces to:
\begin{align*}
    &\nabla_\theta \text{KL}(\pi_\theta(\cdot | S_{i,t}) \| \text{sg}(\pi_\theta(\cdot | T_{i,t}))) \\
    &= -\sum_{\hat{y}_t} \hat{A}_{i,t} \nabla_\theta \pi_\theta(\hat{y}_t | S_{i,t}) \\
    &= -\sum_{\hat{y}_t} \pi_\theta(\hat{y}_t | S_{i,t}) \Big[ \hat{A}_{i,t} \nabla_\theta \log \pi_\theta(\hat{y}_t | S_{i,t}) \Big] \\
    &= -\mathbb{E}_{\hat{y}_t \sim \pi_\theta(\cdot | S_{i,t})} \Big[ \hat{A}_{i,t} \nabla_\theta \log \pi_\theta(\hat{y}_t | S_{i,t}) \Big].
\end{align*}
Therefore:
\begin{align}
    &\nabla_\theta \text{sim}(i,t) = \notag \\
    &\mathbb{E}_{\hat{y}_t \sim \pi_\theta(\cdot | S_{i,t})} \Big[ \hat{A}_{i,t} \nabla_\theta \log \pi_\theta(\hat{y}_t | S_{i,t}) \Big]. \label{eq:sim_grad}
\end{align}

\paragraph{Combining.} Substituting~\eqref{eq:sim_grad} into~\eqref{eq:crpo_step1}, and noting that the rollouts $y_i$ are sampled from $\pi_\theta(\cdot | x)$, we take the outer expectation:
\begin{align}
    &\nabla_\theta \mathcal{L}_{\text{CRPO}}(\theta) = \mathbb{E}_{y_i \sim \pi_\theta(\cdot | x)} \bigg[ \frac{1}{\tau} \sum_{i,t} c_{i,t} \notag \\
    &\cdot \mathbb{E}_{\hat{y}_t \sim \pi_\theta(\cdot | S_{i,t})} \Big[ \hat{A}_{i,t} \nabla_\theta \log \pi_\theta(\hat{y}_t | S_{i,t}) \Big] \bigg].
\end{align}
Noting that $c_{i,t} \leq 0$ for positive pairs, we equivalently write with a negation:
\begin{align}
    &\nabla_\theta \mathcal{L}_{\text{CRPO}}(\theta) = \mathbb{E}_{y_i \sim \pi_\theta(\cdot | x)} \bigg[ -\frac{1}{\tau} \sum_{i,t} c_{i,t} \notag \\
    &\cdot \mathbb{E}_{\hat{y}_t \sim \pi_\theta(\cdot | S_{i,t})} \Big[ \hat{A}_{i,t} \nabla_\theta \log \pi_\theta(\hat{y}_t | S_{i,t}) \Big] \bigg].
\end{align}
This completes the proof.
\end{proof}

\paragraph{Interpretation.} Proposition~\ref{prop:crpo_gradient} shows that $\nabla_\theta \mathcal{L}_{\text{CRPO}}$ takes the form of a contrastively reweighted per-token policy gradient. At positive pairs ($c_{i,t} \leq 0$), gradient descent on $\mathcal{L}_{\text{CRPO}}$ reduces the KL divergence and distills the teacher's reflective behavior into the student; at negative pairs ($c_{i,t} \geq 0$), it enlarges the KL divergence and pushes the student away from the teacher's exposure-biased predictions. The magnitudes $|c_{i,t}|$ are determined by softmax-normalized similarities, so that positions with higher similarity receive larger weights.

\subsection{Anatomy of CRPO$^{*}$ and Comparison with GRPO}\label{sec:crpo_star_vs_grpo}

This subsection analyzes the combined objective CRPO$^{*}$ introduced in the Methodology and compares it term-by-term with classical GRPO.

\paragraph{Classical GRPO.}
GRPO~\citep{dpsk-r1} optimizes a clipped policy ratio under a group-relative outcome reward, regularized by a KL penalty against a frozen reference $\pi_{\theta_{\text{ref}}}$:
\begin{align}
    &\mathcal{L}_{\text{GRPO}}(\theta) = - \mathbb{E}_{x,\{y_i\}} \bigg[ \tfrac{1}{\mathrm{G}} \sum_{i,t} \min\big( \rho_{i,t}(\theta) \hat{R}_{i,t}, \text{clip}(\rho_{i,t}(\theta), \notag \\ &1-\epsilon_{\text{clip}}, 1+\epsilon_{\text{clip}}) \hat{R}_{i,t} \big) \bigg] + \beta\, \mathbb{E}_{x,\{y_i\}}\!\Big[ \text{KL}\!\big( \pi_\theta \,\big\|\, \pi_{\theta_{\text{ref}}} \big) \Big],
\end{align}
with importance ratio $\rho_{i,t}(\theta) = \pi_\theta(y_{i,t}|S_{i,t}) / \pi_{\theta_{\text{old}}}(y_{i,t}|S_{i,t})$, clip threshold $\epsilon_{\text{clip}}$, group-relative reward $\hat{R}_{i,t} = R_i - \mathrm{mean}(\{R_j\}_{j=1}^{\mathrm{G}})$, and KL coefficient $\beta$. The credit signal is outcome-level (constant in $t$) and the anchor $\pi_{\theta_{\text{ref}}}$ is static.

\paragraph{CRPO$^{*}$ as a regularizer swap.}
Plugging $\mathcal{L}_{\text{CRPO}}$ into $\mathcal{L}_{\text{GRPO}}$:
\begin{align}
    &\mathcal{L}_{\text{CRPO}^{*}}(\theta) = \underbrace{- \mathbb{E}_{x,\{y_i\}} \!\Big[ \tfrac{1}{\mathrm{G}} \sum_{i,t} \min(\rho_{i,t} \hat{R}_{i,t}, \text{clip}(\rho_{i,t}) \hat{R}_{i,t}) \Big]}_{\text{outcome-level RL skeleton (unchanged from GRPO)}} \notag \\
    &\quad + \underbrace{\lambda \cdot \mathcal{L}_{\text{CRPO}}(\theta)}_{\text{replaces } \beta\,\text{KL}(\pi_\theta\|\pi_{\theta_{\text{ref}}})}.
\end{align}
The GRPO skeleton is preserved verbatim; only the regularizer is replaced. Table~\ref{tab:grpo_vs_crpostar} summarizes the resulting differences.

\begin{table}[!htbp]
\centering
\small
\renewcommand{\arraystretch}{1.15}
\begin{tabular}{@{}p{2.4cm}p{2.4cm}p{2.4cm}@{}}
\toprule
\textbf{Aspect} & \textbf{GRPO} & \textbf{CRPO$^{*}$} \\
\midrule
RL signal & clipped PG with group-relative $\hat{R}_{i,t}$ & same as GRPO \\
Credit granularity & outcome-level (scalar, constant in $t$) & logit-level (per-token) \\
KL anchor & frozen reference $\pi_{\theta_{\text{ref}}}$ & dynamic self-teacher $\pi_\theta(\cdot|T_{i,t})$ \\
Per-position weighting & uniform across $t$ & contrastive gate $c_{i,t}$ (sign + magnitude) \\
\bottomrule
\end{tabular}
\caption{Component-wise comparison of GRPO and CRPO$^{*}$.}
\label{tab:grpo_vs_crpostar}
\end{table}

The reference-model KL in GRPO functions only as a passive proximity constraint: it discourages $\pi_\theta$ from drifting away from $\pi_{\theta_{\text{ref}}}$ without indicating a preferred direction. In contrast, $\mathcal{L}_{\text{CRPO}}$ provides a directional signal through three mechanisms. First, it carries logit-level advantages $\hat{A}_{i,t}$ at every position. Second, its gate $c_{i,t}$ takes opposite signs on $\mathcal{P}$ and $\mathcal{N}$, simultaneously pulling the student toward reflective tokens and pushing it away from exposure-biased ones. Third, the softmax weights are normalized across the rollout group, which is consistent with the group-relative reward $\hat{R}_{i,t}$ used in GRPO. Consequently, CRPO$^{*}$ replaces the static, single-anchor regularizer of GRPO with a dynamic, position-aware teacher while preserving the GRPO skeleton in full.

\subsection{Trust-Region Regularized Self-Teacher}\label{sec:trust_region_teacher}

If the self-teacher $\pi_\theta(\cdot \mid T_{i,t})$ is allowed to evolve freely with $\theta$, the regularization target of $\mathcal{L}_{\text{CRPO}}$ changes from one optimization step to the next, which can destabilize training. To stabilize the procedure, we constrain the teacher distribution to lie within a trust region around its initial counterpart $\pi_{\theta_{\text{ref}}}(\cdot \mid T_{i,t})$~\citep{SDPO,schulman2017trustregionpolicyoptimization,peng2019advantageweightedregressionsimplescalable}:
\begin{equation}
    \sum_{i,t} \text{KL}\!\big( q(\cdot \mid T_{i,t}) \,\big\|\, \pi_{\theta_{\text{ref}}}(\cdot \mid T_{i,t}) \big) \leq \epsilon, \quad \epsilon > 0,
\end{equation}
and seek the feasible $q$ that is closest (in the sense of cross-entropy) to the current teacher $\pi_\theta(\cdot \mid T_{i,t})$. Following~\citet{peng2019advantageweightedregressionsimplescalable}:
\begin{align}\begin{split}
    \mathop{\mathrm{arg\,max}}_{q \in \Delta} \;\; & \sum_{i,t} \sum_{\hat{y}_t} q(\hat{y}_t \mid T_{i,t}) \log \frac{\pi_\theta(\hat{y}_t \mid T_{i,t})}{\pi_{\theta_{\text{ref}}}(\hat{y}_t \mid T_{i,t})} \\
    \text{s.t.} \;\; & \sum_{i,t} \text{KL}\!\big( q(\cdot \mid T_{i,t}) \,\big\|\, \pi_{\theta_{\text{ref}}}(\cdot \mid T_{i,t}) \big) \leq \epsilon,
\end{split}\label{eq:trust_region_problem}\end{align}
where $\Delta$ denotes the probability simplex over the vocabulary.

\begin{proposition}[Closed-form regularized self-teacher]\label{prop:trust_region_teacher}
The solution of~\eqref{eq:trust_region_problem} admits the closed form
\begin{align}
    &q^{*}(\hat{y}_t \mid T_{i,t}) \propto \exp\!\Big( (1-\alpha) \log \pi_{\theta_{\text{ref}}}(\hat{y}_t \mid T_{i,t}) \notag\\
    &\qquad\qquad + \alpha \log \pi_\theta(\hat{y}_t \mid T_{i,t}) \Big),
\end{align}
where $\alpha > 0$ is determined by the trust-region radius $\epsilon$ through the Lagrange multiplier of the KL constraint, with $\alpha$ decreasing as $\epsilon$ decreases.
\end{proposition}

\begin{proof}
For notational brevity we omit the conditioning context $T_{i,t}$ and write $\pi_\theta(\hat{y}_t)$, $\pi_{\theta_{\text{ref}}}(\hat{y}_t)$, and $q(\hat{y}_t)$ in the following derivation; the same argument is performed in parallel at every position $(i,t)$. Introducing a multiplier $\mu \geq 0$ for the trust-region constraint and a multiplier $\nu$ for the simplex normalization constraint $\sum_{\hat{y}_t} q(\hat{y}_t) = 1$, the Lagrangian of~\eqref{eq:trust_region_problem} is:
\begin{align*}
    &\mathcal{J}(q, \mu, \nu) = \sum_{\hat{y}_t} q(\hat{y}_t) \log \frac{\pi_\theta(\hat{y}_t)}{\pi_{\theta_{\text{ref}}}(\hat{y}_t)} \\
    &\quad - \mu \!\left( \sum_{\hat{y}_t} q(\hat{y}_t) \log \frac{q(\hat{y}_t)}{\pi_{\theta_{\text{ref}}}(\hat{y}_t)} - \epsilon \right) \\
    &\quad + \nu \!\left( \sum_{\hat{y}_t} q(\hat{y}_t) - 1 \right).
\end{align*}
The stationarity condition $\partial \mathcal{J} / \partial q(\hat{y}_t) = 0$ gives, for every $\hat{y}_t$:
\begin{align*}
    0 = \log \frac{\pi_\theta(\hat{y}_t)}{\pi_{\theta_{\text{ref}}}(\hat{y}_t)} - \mu \!\left( \log \frac{q(\hat{y}_t)}{\pi_{\theta_{\text{ref}}}(\hat{y}_t)} + 1 \right) + \nu.
\end{align*}
Solving for $\log q(\hat{y}_t)$:
\begin{align*}
    \log q(\hat{y}_t) = \log \pi_{\theta_{\text{ref}}}(\hat{y}_t) + \tfrac{1}{\mu} \log \frac{\pi_\theta(\hat{y}_t)}{\pi_{\theta_{\text{ref}}}(\hat{y}_t)} + \tfrac{\nu}{\mu} - 1.
\end{align*}
Letting $\alpha := 1/\mu$ and absorbing the $\hat{y}_t$-independent term $\tfrac{\nu}{\mu} - 1$ into the normalization (since $q$ must lie on the simplex), we obtain:
\begin{align*}
    q^{*}(\hat{y}_t) &\propto \pi_{\theta_{\text{ref}}}(\hat{y}_t) \exp\!\Big( \alpha \log \frac{\pi_\theta(\hat{y}_t)}{\pi_{\theta_{\text{ref}}}(\hat{y}_t)} \Big) \\
    &\propto \exp\!\Big( (1-\alpha) \log \pi_{\theta_{\text{ref}}}(\hat{y}_t) + \alpha \log \pi_\theta(\hat{y}_t) \Big).
\end{align*}
Restoring the conditioning context $T_{i,t}$ yields the claimed form. The multiplier $\mu$ is determined by the trust-region budget $\epsilon$ through complementary slackness, and $\alpha = 1/\mu > 0$ is monotonically increasing in $\epsilon$: a smaller $\epsilon$ tightens the trust region and decreases $\alpha$, so that $q^{*}$ stays closer to $\pi_{\theta_{\text{ref}}}$; in the limit $\epsilon \to \infty$, the constraint becomes slack and $q^{*} \to \pi_\theta$.
\end{proof}

\paragraph{Interpretation in CRPO.}
The optimal trust-region teacher $q^{*}$ is a geometric interpolation between $\pi_{\theta_{\text{ref}}}(\cdot \mid T_{i,t})$ and $\pi_\theta(\cdot \mid T_{i,t})$ in log-probability space, with mixing coefficient $\alpha$ controlled by the trust-region radius. Substituting $q^{*}$ for $\pi_\theta(\cdot \mid T_{i,t})$ in the similarity function yields a stabilized variant of $\mathcal{L}_{\text{CRPO}}$:
\begin{align}
    \widetilde{\text{sim}}(i,t) = -\text{KL}\!\big( \pi_\theta(\cdot \mid S_{i,t}) \,\big\|\, \text{sg}(q^{*}(\cdot \mid T_{i,t})) \big).
\end{align}
The limit $\alpha \to 0$ recovers a frozen reference teacher, while $\alpha \to \infty$ recovers the original CRPO objective.

\paragraph{EMA self-teacher as a parameter-space surrogate.}
The construction of $q^{*}$ enforces the trust region in output-distribution space and requires an additional forward pass through $\pi_{\theta_{\text{ref}}}$. An alternative implementation enforces the trust region in parameter space by maintaining a separate teacher parameter $\theta'$ updated as an exponential moving average (EMA) of the student:
\begin{equation}
    \theta' \leftarrow (1 - \alpha)\,\theta' + \alpha\, \theta, \quad \alpha \in (0, 1],
\end{equation}
initialized as $\theta' = \theta_{\text{ref}}$. The self-teacher in $\mathcal{L}_{\text{CRPO}}$ is then instantiated as $\pi_{\theta'}(\cdot \mid T_{i,t})$, and the corresponding similarity becomes
\begin{align}
    \widetilde{\text{sim}}^{\text{EMA}}(i,t) = -\text{KL}\!\big( \pi_\theta(\cdot \mid S_{i,t}) \,\big\|\, \text{sg}(\pi_{\theta'}(\cdot \mid T_{i,t})) \big).
\end{align}
The two variants share the same role of $\alpha$ and the same limiting behavior, but differ in their cost profiles~\citep{SDPO}. The closed-form $q^{*}$ requires additional log-probability evaluations under $\pi_{\theta_{\text{ref}}}$ but no extra parameter storage, whereas the EMA teacher requires an additional copy of $\theta$ in memory but introduces no extra forward pass. We adopt the EMA teacher in our main experiments because of its lower wall-clock cost in the agentic training setup.

\subsection{Top-$K$ Approximation of KL and Entropy}\label{sec:topk_approximation}

The similarity $\text{sim}(i,t)$ and the entropies $\mathrm{E}^{\text{S}}_{i,t}$ and $\mathrm{E}^{\text{T}}_{i,t}$ are sums over the full vocabulary, which is computationally infeasible to materialize for modern tokenizers with $|\mathcal{V}| \sim 10^5$. Following~\citet{SDPO}, we approximate both quantities by retaining only the top-$K$ tokens predicted by the student and folding the remaining tail into a single aggregated symbol. We now formalize this approximation and bound the resulting bias.

\paragraph{Setup.}
Fix $(i,t)$. Let $\mathcal{V}$ be the vocabulary and
\begin{align}
    \mathcal{V}_K(i,t) := \text{top-}K\big( \pi_\theta(\cdot \mid S_{i,t}) \big) \subseteq \mathcal{V}.
\end{align}
For a distribution $p$ on $\mathcal{V}$, write $p_{\text{tail}} := \sum_{\hat{y}_t \notin \mathcal{V}_K} p(\hat{y}_t)$. Within this subsection we abbreviate the student and teacher distributions as $p_{\text{S}} := \pi_\theta(\cdot \mid S_{i,t})$ and $p_{\text{T}} := \pi_\theta(\cdot \mid T_{i,t})$.

\paragraph{Top-$K$ KL estimator.}
Treating the tail as a single aggregated symbol:
\begin{align}
    \widehat{\text{KL}}_K(p \,\|\, r) := \!\!\sum_{\hat{y}_t \in \mathcal{V}_K(i,t)} \!\! p(\hat{y}_t) \log \tfrac{p(\hat{y}_t)}{r(\hat{y}_t)} + p_{\text{tail}} \log \tfrac{p_{\text{tail}}}{r_{\text{tail}}}.
\end{align}
This avoids materializing the full teacher logits: only $K$ teacher log-probs and the tail mass $r_{\text{tail}}$ are evaluated for $r = p_{\text{T}}$.

\begin{proposition}[Top-$K$ KL is a lower bound on the true KL]\label{prop:topk_kl_bound}
For any $K$ and any two distributions $p, r$ on $\mathcal{V}$ with full support,
\begin{align}
    \widehat{\text{KL}}_K(p \,\|\, r) \leq \text{KL}(p \,\|\, r),
\end{align}
with equality if and only if the tail conditional distributions of $p$ and $r$ coincide.
\end{proposition}

\begin{proof}
Decompose the full KL by splitting the sum over $\mathcal{V}_K(i,t)$ and its complement:
\begin{align*}
    &\text{KL}(p \,\|\, r) = \!\!\sum_{\hat{y}_t \in \mathcal{V}_K(i,t)} \!\!p(\hat{y}_t) \log \tfrac{p(\hat{y}_t)}{r(\hat{y}_t)} + \!\!\sum_{\hat{y}_t \notin \mathcal{V}_K(i,t)} \!\!p(\hat{y}_t) \log \tfrac{p(\hat{y}_t)}{r(\hat{y}_t)}.
\end{align*}
Write each tail term as $p(\hat{y}_t) = p_{\text{tail}} \cdot \tilde{p}(\hat{y}_t)$ and $r(\hat{y}_t) = r_{\text{tail}} \cdot \tilde{r}(\hat{y}_t)$, where $\tilde{p}, \tilde{r}$ are the renormalized conditional distributions on the tail. Then:
\begin{align*}
    &\!\!\!\!\!\!\sum_{\hat{y}_t \notin \mathcal{V}_K(i,t)} \!\!\!\!\!p(\hat{y}_t) \log \tfrac{p(\hat{y}_t)}{r(\hat{y}_t)} = p_{\text{tail}} \log \tfrac{p_{\text{tail}}}{r_{\text{tail}}} + p_{\text{tail}} \, \text{KL}(\tilde{p} \,\|\, \tilde{r}),
\end{align*}
where the cross terms have been collected into the conditional KL. Substituting back:
\begin{align*}
    \text{KL}(p \,\|\, r) = \widehat{\text{KL}}_K(p \,\|\, r) + p_{\text{tail}} \cdot \text{KL}(\tilde{p} \,\|\, \tilde{r}) \geq \widehat{\text{KL}}_K(p \,\|\, r),
\end{align*}
since $p_{\text{tail}} \geq 0$ and $\text{KL}(\tilde{p} \,\|\, \tilde{r}) \geq 0$ by Gibbs' inequality, with equality if and only if $\tilde{p} = \tilde{r}$.
\end{proof}

\paragraph{Discussion.}
The residual $\text{KL}(p \,\|\, r) - \widehat{\text{KL}}_K(p \,\|\, r) = p_{\text{tail}} \cdot \text{KL}(\tilde p \,\|\, \tilde r)$ factorizes into two small terms. Empirically $p_{\text{tail}} \leq 10^{-2}$ at $K = 100$ for modern LLMs, and $\text{KL}(\tilde p \,\|\, \tilde r)$ is bounded since both tails spread over low-probability tokens. The surrogate is therefore a tight strict lower bound and never overestimates the regularization signal.

\paragraph{Top-$K$ entropy estimator.}
Applying the same folding to the entropy: for any $p$ on $\mathcal{V}$,
\begin{align}
    \widehat{\mathrm{H}}_K(p) := -\!\!\sum_{\hat{y}_t \in \mathcal{V}_K(i,t)} \!\!p(\hat{y}_t) \log p(\hat{y}_t) - p_{\text{tail}} \log p_{\text{tail}}.
\end{align}
This treats the tail as a single aggregated symbol of mass $p_{\text{tail}}$. Instantiating $p = \pi_\theta(\cdot \mid S_{i,t})$ gives $\widehat{\mathrm{E}}^{\text{S},K}_{i,t}$, and $p = \pi_\theta(\cdot \mid T_{i,t})$ gives $\widehat{\mathrm{E}}^{\text{T},K}_{i,t}$, with $\widehat{\Delta \mathrm{E}}_{i,t} := \widehat{\mathrm{E}}^{\text{S},K}_{i,t} - \widehat{\mathrm{E}}^{\text{T},K}_{i,t}$.

\begin{proposition}[Top-$K$ entropy is a lower bound on the true entropy]\label{prop:topk_entropy_bound}
For any $K$ and any distribution $p$ on $\mathcal{V}$,
\begin{align}
    \widehat{\mathrm{H}}_K(p) \leq \mathrm{H}(p),
\end{align}
with equality if and only if the tail conditional $\tilde{p}$ degenerates to a point mass.
\end{proposition}

\begin{proof}
Split the entropy along $\mathcal{V}_K(i,t)$ and its complement, and write tail terms as $p(\hat{y}_t) = p_{\text{tail}} \tilde{p}(\hat{y}_t)$:
\begin{align*}
    &\mathrm{H}(p) = -\!\!\sum_{\hat{y}_t \in \mathcal{V}_K(i,t)} \!\!p(\hat{y}_t) \log p(\hat{y}_t) - \!\!\sum_{\hat{y}_t \notin \mathcal{V}_K(i,t)} \!\!p(\hat{y}_t) \log p(\hat{y}_t) \\
    &= -\!\!\sum_{\hat{y}_t \in \mathcal{V}_K(i,t)} \!\!p(\hat{y}_t) \log p(\hat{y}_t) - p_{\text{tail}} \log p_{\text{tail}} + p_{\text{tail}} \mathrm{H}(\tilde{p}) \\
    &= \widehat{\mathrm{H}}_K(p) + p_{\text{tail}} \mathrm{H}(\tilde{p}) \geq \widehat{\mathrm{H}}_K(p),
\end{align*}
since $\mathrm{H}(\tilde{p}) \geq 0$, with equality if and only if $\tilde{p}$ is a point mass (zero entropy).
\end{proof}

\paragraph{Bias cancellation in $\widehat{\Delta \mathrm{E}}_{i,t}$.}
The per-distribution residual $p_{\text{tail}} \mathrm{H}(\tilde p)$ enters both $\mathrm{E}^{\text{S}}_{i,t}$ and $\mathrm{E}^{\text{T}}_{i,t}$ with the same sign, and the two residuals largely cancel in the difference $\widehat{\Delta \mathrm{E}}_{i,t}$ when the student and the teacher have comparable tail masses (which holds in practice, since $\mathcal{V}_K(i,t)$ is fixed by the student and the teacher shares the same backbone). Moreover, the partition $\mathcal{P}, \mathcal{N}$ depends only on the rank of $\Delta \mathrm{E}_{i,t}$ within the rollout group through the Bottom-$p\%$ operation. Hence $\widehat{\Delta \mathrm{E}}_{i,t}$ need only preserve this rank, rather than reproduce the absolute values, to recover the correct contrastive partition.

\section{More Results}\label{sec:more_results}

\subsection{Statistical Analysis of Entropy Collapse After Tool Calls}\label{sec:entropy_collapse_stats}

Figure~\ref{fig: intro1} in the Introduction illustrates a single trajectory where the self-teacher's entropy collapses after tool calls. To validate that this phenomenon is systematic rather than anecdotal, we conduct a statistical analysis across all rollouts in the GAIA training set (Qwen3-8B backbone, 1K questions $\times$ 8 rollouts).

\paragraph{Setup.} For each position $t$ in every rollout, we compute the entropy gap $\Delta \mathrm{E}_{i,t} = \mathrm{E}^{\text{S}}_{i,t} - \mathrm{E}^{\text{T}}_{i,t}$ and record its distance (in tokens) to the nearest preceding tool-call boundary. We define a position as exhibiting \emph{significant teacher entropy collapse} when $\Delta \mathrm{E}_{i,t} > 0.3$ (i.e., the teacher's entropy is more than 0.3 nat below the student's, corresponding to a perplexity ratio $\exp(0.3) \approx 1.35\times$). We then bin positions by their distance to the nearest tool-call boundary and report (a)~the proportion of positions exceeding the threshold in each bin, and (b)~the full distribution of entropy gaps grouped into four distance ranges.

\paragraph{Results.} Figure~\ref{fig:entropy_stats} presents the results. Notably, the entropy collapse does \emph{not} occur immediately after a tool call. Panel~(a) shows that in the first 10 tokens (the ``react'' phase), only 31\% of positions exceed the threshold---the teacher is still processing the newly injected information. The proportion then surges to \textbf{72\%--76\%} in the 10--30 token window (the ``copy'' phase), where the teacher begins reproducing reasoning routes from the privileged demonstrations, causing its entropy to collapse. Beyond 50 tokens, the proportion drops back to 14\% as the teacher's advantage dissipates. Panel~(b) visualizes this pattern via violin plots: the copy phase (10--30 tokens) exhibits a mean gap of ${\sim}0.50$ with heavy concentration above the threshold, while the react phase (0--10 tokens) is moderate and positions far from tool calls ($>$50 tokens) are centered near zero.

\begin{figure}[!htbp]
    \centering
    \includegraphics[width=\linewidth]{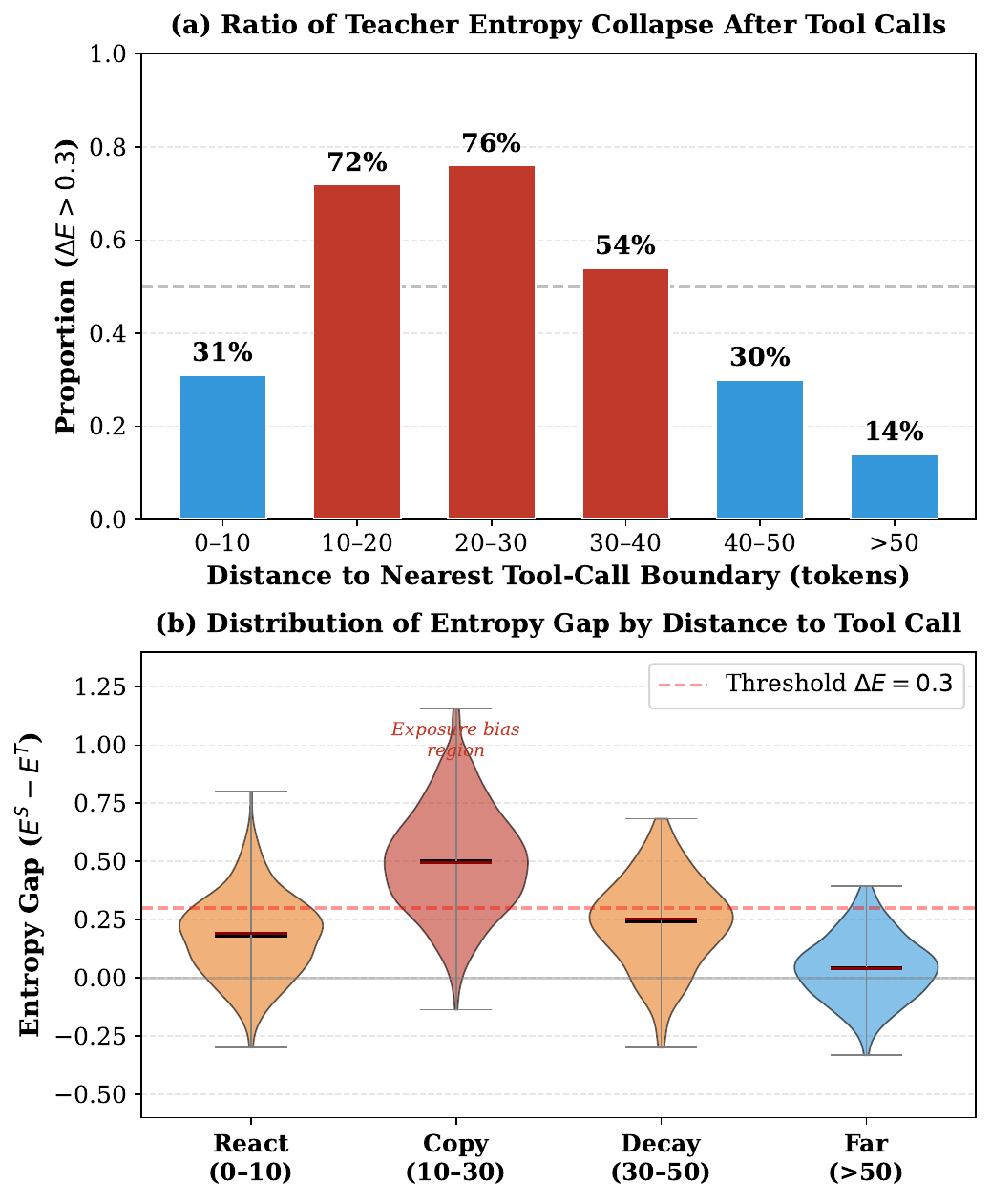}
    \caption{Statistical analysis of teacher entropy collapse after tool calls across all GAIA training rollouts (Qwen3-8B, 8K trajectories). (a)~Proportion of positions where $\Delta \mathrm{E} > 0.3$, binned by distance to nearest tool-call boundary. (b)~Violin plot of entropy gap distribution grouped by distance range.}
    \label{fig:entropy_stats}
\end{figure}

\paragraph{Analysis.} These statistics reveal that the exposure bias of the self-teacher is \emph{spatially concentrated} in a delayed window after tool-call boundaries: the teacher first processes the new context (react phase, moderate gap), then begins copying privileged reasoning routes (copy phase, large gap), before gradually returning to baseline. This delayed-peak pattern further justifies the entropy-gap-based judger in CRPO: it naturally captures the copy phase as negative pairs without requiring explicit distance heuristics. The contrastive objective actively pushes the student away from the teacher's overconfident route-convergence predictions at these positions, while preserving distillation at react-phase positions where the teacher engages in genuine reflective exploration.

\paragraph{Token-level evidence: what does the teacher say at collapse positions?}
To further verify that the entropy collapse stems from the teacher copying privileged reasoning routes rather than performing independent reasoning, we visualize the most frequent tokens at high-$\Delta \mathrm{E}$ and low-$\Delta \mathrm{E}$ positions in Figure~\ref{fig:wordcloud}. At positions where the teacher's entropy collapses (panel~a), the dominant tokens are search-related action words (\texttt{search}, \texttt{find}, \texttt{query}, \texttt{information}, \texttt{determine}, \texttt{relevant}, \texttt{Wikipedia}) and factual anchors (\texttt{found}, \texttt{identified}, \texttt{provided}, \texttt{conclude})---precisely the vocabulary that characterizes the demonstrated search trajectories in the privileged context. In contrast, at positions where the teacher maintains high entropy or even exceeds the student (panel~b), the dominant tokens are reasoning primitives (\texttt{think}, \texttt{let}, \texttt{solve}, \texttt{check}, \texttt{calculate}) and mathematical notation (\texttt{frac}, \texttt{boxed}, \texttt{sqrt}, \texttt{python}), reflecting genuine reflective exploration and computation. This lexical divergence confirms our hypothesis: the teacher's entropy collapse is driven by \emph{route convergence}---copying the search-and-conclude patterns from privileged demonstrations---rather than by legitimate uncertainty reduction through reasoning.

\begin{figure}[t]
    \centering
    \includegraphics[width=\linewidth]{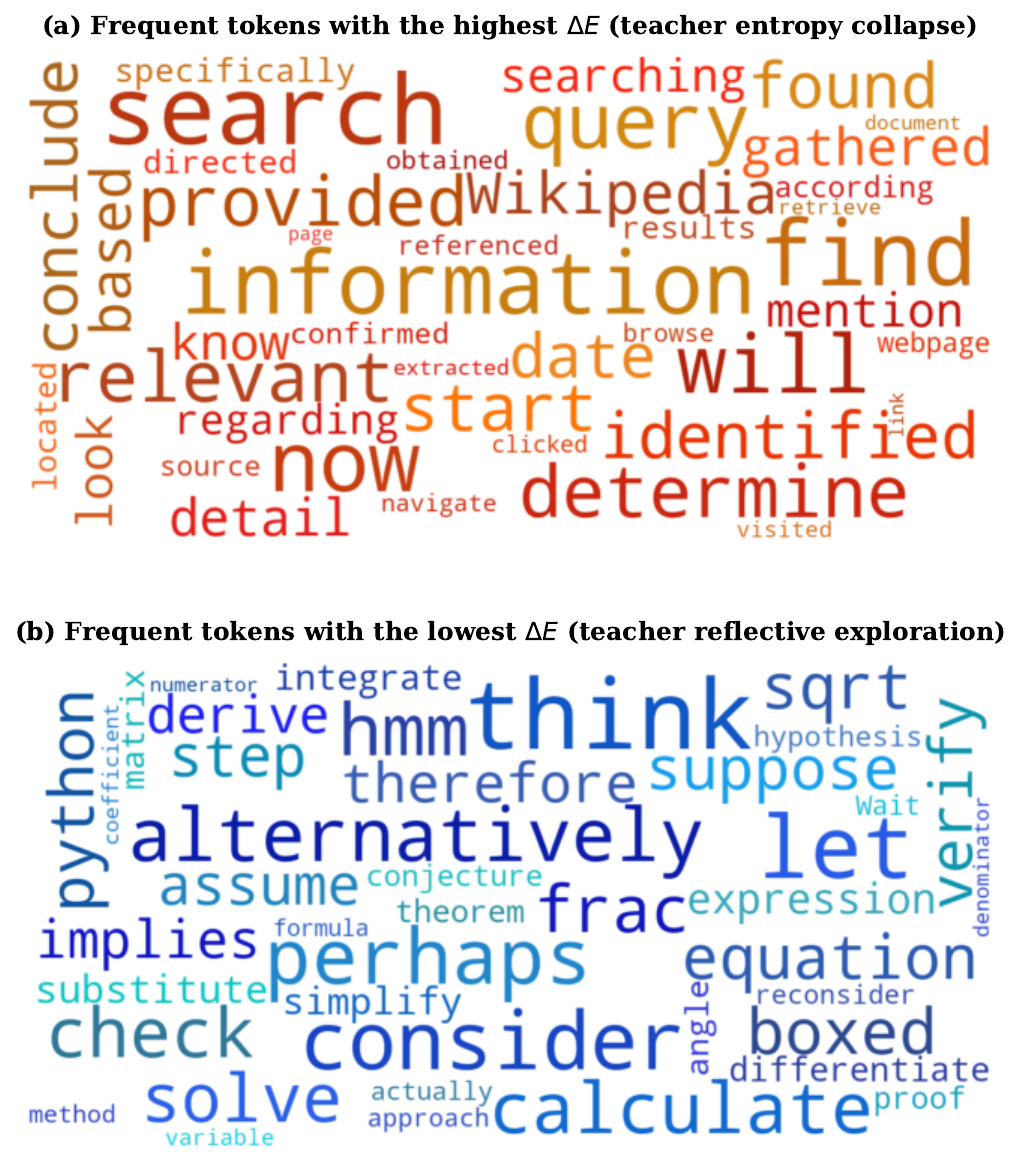}
    \caption{Word clouds of frequent tokens at (a)~high-$\Delta \mathrm{E}$ positions (teacher entropy collapse) and (b)~low-$\Delta \mathrm{E}$ positions (teacher reflective exploration). The collapse positions are dominated by search-related action tokens from the privileged demonstrations.}
    \label{fig:wordcloud}
\end{figure}

\subsection{Ablation: Role of Negative Pairs in Contrastive Objective}\label{sec:ablation_negative}

A natural question is whether the performance gain of CRPO stems from selectively distilling positive positions alone, or whether actively pushing the student \emph{away} from negative positions provides additional benefit. To answer this, we compare two variants on the four deep-search benchmarks using Qwen3-8B:

\begin{itemize}[leftmargin=1em]
    \item \textbf{Positive Only}: after the entropy-gap judger classifies positions into $\mathcal{P}$ and $\mathcal{N}$, we only minimize the KL divergence at positive positions $\mathcal{P}$ and completely ignore the negative set $\mathcal{N}$ (no gradient signal from them).
    \item \textbf{CRPO (Full)}: our full contrastive objective that simultaneously pulls the student toward the teacher at $\mathcal{P}$ and pushes it away at $\mathcal{N}$ via the InfoNCE loss.
\end{itemize}

\paragraph{Results.} Figure~\ref{fig:ablation_negative} reports the comparison. The full CRPO consistently outperforms the positive-only variant across all four benchmarks: $+3.9\%$ on GAIA, $+4.0\%$ on WebWalkerQA, $+1.4\%$ on HLE, and $+4.0\%$ on xbench. While positive-only distillation already improves over uniform OPSD by filtering out exposure-biased positions, it leaves the model vulnerable to drifting toward the teacher's overconfident predictions at negative positions through the implicit policy gradient. The negative-pair repulsion in CRPO provides an explicit corrective signal that actively discourages route convergence, yielding a meaningful and consistent improvement.

\begin{figure}[!htbp]
    \centering
    \includegraphics[width=\linewidth]{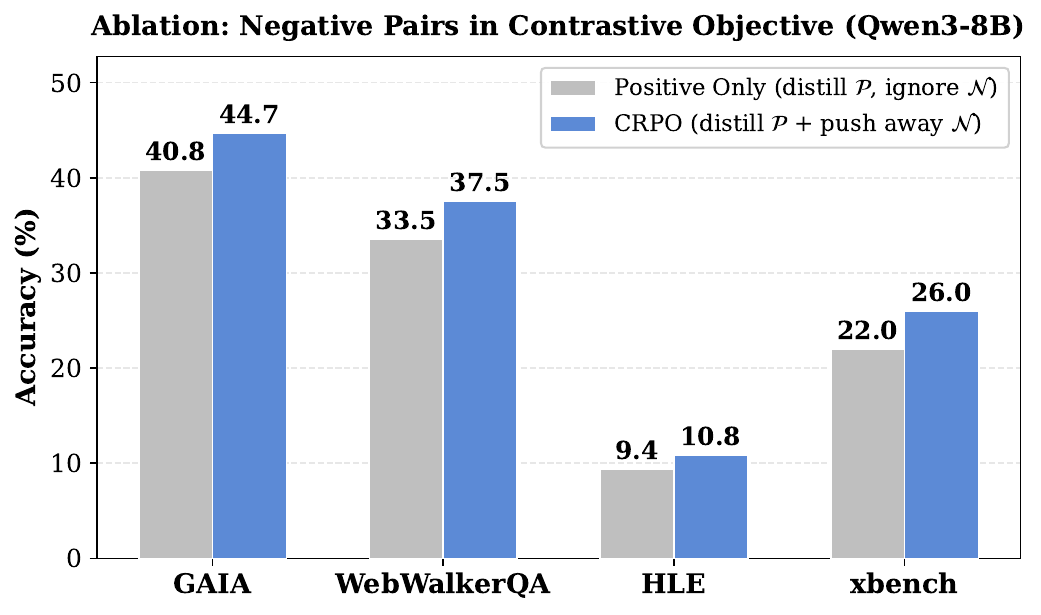}
    \caption{Ablation study on the role of negative pairs in CRPO (Qwen3-8B). ``Positive Only'' distills only at $\mathcal{P}$ positions; ``CRPO'' additionally pushes the student away from $\mathcal{N}$ positions via the contrastive loss.}
    \label{fig:ablation_negative}
\end{figure}

\subsection{Sensitivity to Global Rollout Size}\label{sec:rollout_sensitivity}

The group-wise contrastive mechanism in CRPO ranks entropy gaps across all positions of $\mathrm{G}$ rollouts to determine positive and negative pairs. A larger $\mathrm{G}$ provides more positions for ranking, yielding a more reliable partition between $\mathcal{P}$ and $\mathcal{N}$. We investigate how sensitive CRPO is to the global rollout size by varying $\mathrm{G} \in \{1, 2, 4, 8, 16\}$ on Qwen3-8B and reporting the average accuracy across the four deep-search benchmarks (GAIA, WebWalkerQA, HLE, xbench).

\paragraph{Results.} Figure~\ref{fig:rollout_sensitivity} presents the results. Performance improves substantially from $\mathrm{G}=1$ (24.8\%) to $\mathrm{G}=4$ (32.4\%), with a marginal gain when further increasing to $\mathrm{G}=8$ (34.2\%) and $\mathrm{G}=16$ (34.9\%). The curve exhibits a clear diminishing-return pattern: doubling from 1 to 2 yields $+3.8\%$, from 2 to 4 yields $+3.8\%$, from 4 to 8 yields $+1.8\%$, and from 8 to 16 only $+0.7\%$.

\paragraph{Analysis.} All variants are trained for the same number of optimization steps, so the performance differences are solely attributable to the quality of the contrastive signal under different rollout budgets. The rapid improvement at small $\mathrm{G}$ confirms that the group-wise ranking mechanism requires a sufficient number of positions to reliably separate positive from negative pairs---with only one rollout, the ranking degenerates into a within-sequence comparison that lacks diversity, producing a noisy $\mathcal{P}/\mathcal{N}$ partition that weakens both the distillation and the repulsion signal. The saturation beyond $\mathrm{G}=8$ indicates that the contrastive partition is already statistically stable once the position pool is large enough (${\sim}8 \times c_i$ positions), and additional rollouts contribute diminishing marginal information for the entropy-gap ranking. This aligns with our default setting of $\mathrm{G}=8$ in the deep-search experiments, which achieves near-optimal performance at the same training cost per step as smaller rollout budgets (the only difference is the number of sampled trajectories per question, not the total training steps).

\begin{figure}[!htbp]
    \centering
    \includegraphics[width=0.85\linewidth]{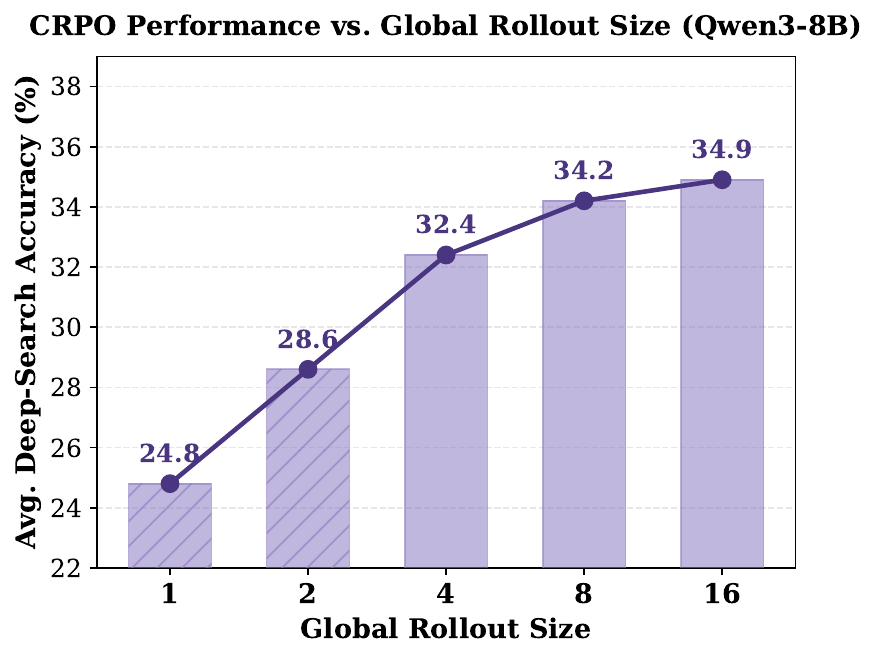}
    \caption{CRPO accuracy (average over four deep-search benchmarks) as a function of global rollout size $\mathrm{G}$ on Qwen3-8B. Performance saturates around $\mathrm{G}=8$.}
    \label{fig:rollout_sensitivity}
\end{figure}

\section{Algorithm Pseudocodes}\label{sec:algorithm_pseudocodes}

For completeness and ease of comparison, we summarize CRPO and CRPO$^{*}$ in Algorithm~\ref{alg:crpo}, and provide pseudocodes of the reference algorithms that CRPO directly builds on or compares against: GRPO~\citep{dpsk-r1} in Algorithm~\ref{alg:grpo} as the representative outcome-level RL baseline, on-policy self-distillation (OPSD/SDPO~\citep{OPSD,SDPO}) in Algorithm~\ref{alg:opsd} as the self-distillation predecessor that CRPO generalizes, and RLSD~\citep{RLSD} in Algorithm~\ref{alg:rlsd} as the hybrid baseline that re-weights GRPO advantages with the self-teacher gap. All procedures are presented under a unified notation consistent with the Methodology section.

\begin{algorithm}[!htbp]
\caption{CRPO / CRPO$^{*}$}
\label{alg:crpo}
\small
\begin{algorithmic}[1]
\REQUIRE Policy $\pi_\theta$; dataset $\mathcal{D}$; rollouts per question $\mathrm{G}$; positive ratio $p\%$; temperature $\tau$; (CRPO$^{*}$ only) coefficient $\lambda$.
\REPEAT
\STATE Sample question $x \sim \mathcal{D}$ and rollouts $\{y_i\}_{i=1}^{\mathrm{G}} \sim \pi_\theta(\cdot \mid x)$.
\STATE Evaluate $\{y_i\}$ to obtain rewards $\{R_i\}$ and feedback $\{f_i\}$.
\STATE \textit{// Build student and self-teacher contexts.}
\STATE Form $S_{i,t} = (x, y_{i,<t})$ and $T_{i,t} = (x, f_i, y_{i,<t})$.
\STATE \textit{// Judger: identify positive and negative pairs.}
\STATE Compute entropies $\mathrm{E}^{\text{S}}_{i,t}$, $\mathrm{E}^{\text{T}}_{i,t}$ and the difference $\Delta \mathrm{E}_{i,t} = \mathrm{E}^{\text{S}}_{i,t} - \mathrm{E}^{\text{T}}_{i,t}$.
\STATE Rank $\{\Delta \mathrm{E}_{i,t}\}$ within the rollout group; assign $(i,t)$ with the bottom $p\%$ to $\mathcal{P}$ and the rest to $\mathcal{N}$.
\STATE \textit{// Contrastive self-distillation.}
\STATE Evaluate $\text{sim}(i,t) = -\text{KL}(\pi_\theta(\cdot \mid S_{i,t}) \,\|\, \text{sg}(\pi_\theta(\cdot \mid T_{i,t})))$ for all $(i,t) \in \mathcal{P} \cup \mathcal{N}$.
\STATE Compute the InfoNCE loss $\mathcal{L}_{\text{CRPO}}(\theta)$ from $\{\text{sim}(i,t)\}$, $\mathcal{P}$, $\mathcal{N}$, and $\tau$.
\STATE \textit{// Optimization step.}
\IF{CRPO$^{*}$}
    \STATE $\mathcal{L}(\theta) \leftarrow \mathcal{L}_{\text{GRPO}}(\theta;\{R_i\}) + \lambda\, \mathcal{L}_{\text{CRPO}}(\theta)$
\ELSE
    \STATE $\mathcal{L}(\theta) \leftarrow \mathcal{L}_{\text{CRPO}}(\theta)$
\ENDIF
\STATE Update $\theta$ by gradient descent on $\mathcal{L}(\theta)$.
\UNTIL{converged}
\end{algorithmic}
\end{algorithm}

\begin{algorithm}[!htbp]
\caption{GRPO~\citep{dpsk-r1}}
\label{alg:grpo}
\small
\begin{algorithmic}[1]
\REQUIRE Policy $\pi_\theta$; reference policy $\pi_{\theta_{\text{ref}}}$; dataset $\mathcal{D}$; rollouts per question $\mathrm{G}$; clip threshold $\epsilon_{\text{clip}}$; KL coefficient $\beta$.
\REPEAT
\STATE Sample question $x \sim \mathcal{D}$ and rollouts $\{y_i\}_{i=1}^{\mathrm{G}} \sim \pi_{\theta_{\text{old}}}(\cdot \mid x)$.
\STATE Evaluate $\{y_i\}$ to obtain rewards $\{R_i\}_{i=1}^{\mathrm{G}}$.
\STATE \textit{// Group-relative advantage.}
\STATE Compute $\hat{R}_{i,t} = R_i - \mathrm{mean}(\{R_j\}_{j=1}^{\mathrm{G}})$ (constant in $t$).
\STATE \textit{// Clipped policy gradient with reference KL.}
\STATE Compute importance ratios $\rho_{i,t}(\theta) = \pi_\theta(y_{i,t} \mid S_{i,t}) / \pi_{\theta_{\text{old}}}(y_{i,t} \mid S_{i,t})$.
\STATE Compute clipped policy-gradient term
\STATE \quad $\begin{aligned}[t]
    \mathcal{L}_{\text{pg}}(\theta) = -\,\mathbb{E}\Big[\tfrac{1}{\mathrm{G}}\sum_{i,t} \min\!\big(&\rho_{i,t}\hat{R}_{i,t},\, \\
    \text{clip}(\rho_{i,t}, &1-\epsilon_{\text{clip}}, 1+\epsilon_{\text{clip}})\hat{R}_{i,t}\big)\Big].
\end{aligned}$
\STATE Form total loss $\mathcal{L}_{\text{GRPO}}(\theta) = \mathcal{L}_{\text{pg}}(\theta) + \beta\,\mathbb{E}[\text{KL}(\pi_\theta \,\|\, \pi_{\theta_{\text{ref}}})]$.
\STATE Update $\theta$ by gradient descent on $\mathcal{L}_{\text{GRPO}}(\theta)$; set $\pi_{\theta_{\text{old}}} \leftarrow \pi_\theta$.
\UNTIL{converged}
\end{algorithmic}
\end{algorithm}

\begin{algorithm}[!htbp]
\caption{OPSD/SDPO~\citep{OPSD,SDPO}}
\label{alg:opsd}
\small
\begin{algorithmic}[1]
\REQUIRE Policy $\pi_\theta$; dataset $\mathcal{D}$; rollouts per question $\mathrm{G}$.
\REPEAT
\STATE Sample question $x \sim \mathcal{D}$ and rollouts $\{y_i\}_{i=1}^{\mathrm{G}} \sim \pi_\theta(\cdot \mid x)$.
\STATE Evaluate $\{y_i\}$ to obtain feedback $\{f_i\}$.
\STATE \textit{// Build student and self-teacher contexts (same policy, two views).}
\STATE Form $S_{i,t} = (x, y_{i,<t})$ and $T_{i,t} = (x, f_i, y_{i,<t})$.
\STATE \textit{// Uniform self-distillation across all positions.}
\STATE Compute per-position $\text{KL}(\pi_\theta(\cdot \mid S_{i,t}) \,\|\, \text{sg}(\pi_\theta(\cdot \mid T_{i,t})))$ for all $(i,t)$.
\STATE Form $\mathcal{L}_{\text{OPSD}}(\theta) = \tfrac{1}{\mathrm{G}}\sum_{i,t} \text{KL}(\pi_\theta(\cdot \mid S_{i,t}) \,\|\, \text{sg}(\pi_\theta(\cdot \mid T_{i,t})))$.
\STATE Update $\theta$ by gradient descent on $\mathcal{L}_{\text{OPSD}}(\theta)$.
\UNTIL{converged}
\end{algorithmic}
\end{algorithm}

\begin{algorithm}[!htbp]
\caption{RLSD~\citep{RLSD}}
\label{alg:rlsd}
\small
\begin{algorithmic}[1]
\REQUIRE Policy $\pi_\theta$; reference policy $\pi_{\theta_{\text{ref}}}$; dataset $\mathcal{D}$; rollouts per question $\mathrm{G}$; mixing coefficient $\lambda$; weight clip bound $\epsilon_w$; policy clip threshold $\epsilon_{\text{clip}}$.
\REPEAT
\STATE Sample question $x \sim \mathcal{D}$ and rollouts $\{y_i\}_{i=1}^{\mathrm{G}} \sim \pi_{\theta_{\text{old}}}(\cdot \mid x)$.
\STATE Evaluate $\{y_i\}$ to obtain rewards $\{R_i\}_{i=1}^{\mathrm{G}}$ and feedback $\{f_i\}$.
\STATE \textit{// Build student and self-teacher contexts.}
\STATE Form $S_{i,t} = (x, y_{i,<t})$ and $T_{i,t} = (x, f_i, y_{i,<t})$.
\STATE \textit{// Group-relative advantage.}
\STATE Compute $\hat{R}_{i} = R_i - \mathrm{mean}(\{R_j\}_{j=1}^{\mathrm{G}})$.
\STATE \textit{// Token-level advantage reweighting via self-teacher gap.}
\STATE Evaluate teacher log-probs $\log\pi_\theta(y_{i,t} \mid T_{i,t})$ via forward pass with the augmented context.
\STATE Compute the teacher--student gap $\delta_{i,t} = \log\pi_\theta(y_{i,t} \mid T_{i,t}) - \log\pi_{\theta_{\text{old}}}(y_{i,t} \mid S_{i,t})$.
\STATE Compute reweighting factor $w_{i,t} = \text{clip}\!\big(\exp(\text{sign}(\hat{R}_{i}) \cdot \delta_{i,t}),\, 1{-}\epsilon_w,\, 1{+}\epsilon_w\big)$.
\STATE Form token-level advantage $\hat{A}_{i,t} = \hat{R}_{i} \cdot \big[(1-\lambda) + \lambda\, w_{i,t}\big]$.
\STATE \textit{// Clipped surrogate with token-level advantages.}
\STATE Compute importance ratios $\rho_{i,t}(\theta) = \pi_\theta(y_{i,t} \mid S_{i,t}) / \pi_{\theta_{\text{old}}}(y_{i,t} \mid S_{i,t})$.
\STATE Form $\mathcal{L}_{\text{RLSD}}(\theta) = -\,\mathbb{E}\Big[\tfrac{1}{\mathrm{G}}\sum_{i,t} \min\!\big(\rho_{i,t}\hat{A}_{i,t},\, \text{clip}(\rho_{i,t}, 1{-}\epsilon_{\text{clip}}, 1{+}\epsilon_{\text{clip}})\hat{A}_{i,t}\big)\Big]$.
\STATE Update $\theta$ by gradient descent on $\mathcal{L}_{\text{RLSD}}(\theta)$; set $\pi_{\theta_{\text{old}}} \leftarrow \pi_\theta$.
\UNTIL{converged}
\end{algorithmic}
\end{algorithm}

\end{document}